\documentclass{article}
\usepackage{amsmath}
\usepackage{amssymb}
\usepackage{mathtools}
\usepackage{amsthm}
\usepackage{algorithm}
\usepackage{algorithmicx}
\usepackage{algpseudocode}

\usepackage{xcolor}
\definecolor{darkblue}{rgb}{0,0,0.5}
\definecolor{turquoise}{RGB}{64, 224, 208}

\usepackage[breaklinks,colorlinks, linkcolor=darkblue, urlcolor=darkblue, citecolor=darkblue]{hyperref}
\usepackage[utf8]{inputenc} % allow utf-8 input
\usepackage[T1]{fontenc}    % use 8-bit T1 fonts
\usepackage{url}            % simple URL typesetting
\usepackage{booktabs}       % professional-quality tables
\usepackage{amsfonts}       % blackboard math symbols
\usepackage{nicefrac}       % compact symbols for 1/2, etc.
\usepackage{microtype}      % microtypography
\usepackage{xcolor}         % colors
\usepackage{graphicx}
\usepackage{standalone}
\usepackage{tabularx}
\usepackage{array}
\usepackage{mathtools}
\usepackage{amsmath,amsfonts,amssymb,amsthm}
\usepackage{subcaption}
\usepackage{multirow}
\usepackage{caption}
\usepackage{makecell}
\usepackage[pro]{fontawesome5}

\DeclareFixedFont{\ttb}{T1}{txtt}{bx}{n}{12} % for bold
\DeclareFixedFont{\ttm}{T1}{txtt}{m}{n}{12}  % for normal

\usepackage{color}
\definecolor{deepblue}{rgb}{0,0,0.5}
\definecolor{deepred}{rgb}{0.6,0,0}
\definecolor{deepgreen}{rgb}{0,0.5,0}

\usepackage{listings}

\newcommand\pythonstyle{\lstset{
language=Python,
basicstyle=\ttfamily,
morekeywords={self, vn_entropy},              % Add keywords here
keywordstyle=\ttfamily\color{deepblue},
emph={MyClass,__init__},          % Custom highlighting
emphstyle=\ttfamily\color{deepred},    % Custom highlighting style
commentstyle=\ttfamily\color{olive},
stringstyle=\color{deepgreen},
frame=tb,                         % Any extra options here
showstringspaces=false
}}

\lstnewenvironment{python}[1][]
{
\pythonstyle
\lstset{#1} 
}{}

\newcommand\pythoninline[1]{{\pythonstyle\lstinline!#1!}}

\usepackage{pifont}
\usepackage{xspace}

\usepackage{tikz,pgfplots}
\usetikzlibrary{plotmarks,shapes,shapes.arrows,positioning}
\pgfplotsset{compat=newest} 
\pgfkeys{/pgf/number format/.cd,1000 sep={}}

\usepackage{commath}
\usepackage{enumitem}
\usepackage{systeme}
\usepackage{accents}
\usepackage{yfonts}
\usepackage{appendix}
\usepackage{svg}
\usepackage{bm}
\usepackage{xr}
\usepackage{pgfplots}
\usepackage{adjustbox}
\usepackage{dirtytalk}
\usepackage{mdframed}
\usepackage{wrapfig}
\usepackage{algorithm}
\usepackage{multirow}

\usepackage[utf8]{inputenc} % allow utf-8 input
\usepackage[T1]{fontenc}    % use 8-bit T1 fonts
\usepackage{url}            % simple URL typesetting
\usepackage{booktabs}       % professional-quality tables
\usepackage{amsfonts}       % blackboard math symbols
\usepackage{nicefrac}       % compact symbols for 1/2, etc.
\usepackage{microtype}      % microtypography
\usepackage{xcolor}

\usepackage{tikz,pgfplots}
\usetikzlibrary{plotmarks,shapes,shapes.arrows,positioning}
\pgfplotsset{compat=newest} 
\pgfkeys{/pgf/number format/.cd,1000 sep={}}

\usepgfplotslibrary{groupplots}

\DeclareMathOperator{\Tr}{Tr}

\makeatletter
\newcommand*{\addFileDependency}[1]{% argument=file name and extension
  \typeout{(#1)}
  \@addtofilelist{#1}
  \IfFileExists{#1}{}{\typeout{No file #1.}}
}
\makeatother

\usepackage[capitalize,nameinlink]{cleveref}
\crefname{section}{Sec.}{Secs.}
\crefname{proposition}{Prop.}{Props.}
\crefname{lemma}{Lem.}{Lems.}
\crefname{model}{Mod.}{Mods.}
\crefname{appendix}{App.}{Apps.}
\crefname{theorem}{Thm.}{Thms.}

\theoremstyle{plain}
\newtheorem{theorem}{Theorem}[section]
\newtheorem{proposition}[theorem]{Proposition}
\newtheorem{lemma}[theorem]{Lemma}

\newtheorem{definition}[theorem]{Definition}

\theoremstyle{remark}

\graphicspath{ {images/} }

\newenvironment{FrameProposition}
  {\begin{mdframed}[linecolor=black!5,backgroundcolor=black!5,roundcorner=3pt,innerleftmargin=4pt,innerrightmargin=4pt,innertopmargin=6pt,innerbottommargin=6pt]\begin{proposition}}
  {\end{proposition}\end{mdframed}}

\newenvironment{FrameTheorem}
  {\begin{mdframed}[linecolor=black!5,backgroundcolor=black!5,roundcorner=3pt,innerleftmargin=4pt,innerrightmargin=4pt,innertopmargin=6pt,innerbottommargin=6pt]\begin{theorem}}
  {\end{theorem}\end{mdframed}}

\setlist{topsep=0pt}

\let\oldtextbf\textbf
\renewcommand{\textbf}[1]{\oldtextbf{\boldmath #1}}

\newcommand*{\transposed}{^\top}
\newcommand{\R}{\mathbb{R}}
\newcommand{\VNE}{\operatorname{VNE}}
\newcommand{\KLE}{\operatorname{KLE}}
\newcommand{\SE}{\operatorname{SE}}
\definecolor{applegreen}{rgb}{0.55, 0.71, 0.0}

\newcommand{\nipstitle}[1]{{%
    \def\toptitlebar{\hrule height4pt \vskip .25in \vskip -\parskip} 
    \def\bottomtitlebar{\vskip .29in \vskip -\parskip \hrule height1pt \vskip .09in} 
    \phantomsection\hsize\textwidth\linewidth\hsize%
    \vskip 0.1in%
    \toptitlebar%
    \begin{minipage}{\textwidth}%
        \centering{\LARGE\bf #1\par}%
    \end{minipage}%
    \bottomtitlebar%
    \addcontentsline{toc}{section}{#1}%
}}

\usepackage{comment}

\usepackage[nonatbib, preprint]{neurips_2024}
\usepackage[numbers]{natbib}

\renewcommand{\paragraph}[1]{\textbf{#1}~~}

\usepackage[utf8]{inputenc} % allow utf-8 input
\usepackage[T1]{fontenc}    % use 8-bit T1 fonts
\usepackage{hyperref}       % hyperlinks
\usepackage{url}            % simple URL typesetting
\usepackage{booktabs}       % professional-quality tables
\usepackage{amsfonts}       % blackboard math symbols
\usepackage{nicefrac}       % compact symbols for 1/2, etc.
\usepackage{microtype}      % microtypography
\usepackage{xcolor}         % colors

\title{Kernel Language Entropy: Fine-grained Uncertainty Quantification for LLMs from Semantic Similarities}

\newcommand{\sspace}{\hspace{10pt}}

\author{%
  Alexander Nikitin$^{1}$
  \sspace
  Jannik Kossen$^{2}$
  \sspace
  Yarin Gal$^{2}$
  \sspace
  Pekka Marttinen$^{1}$
\\[.8em]
$^1$ Department of Computer Science, Aalto University \\
$^2$ OATML, Department of Computer Science, University of Oxford\\
\texttt{alexander.nikitin@aalto.fi}\\
 \phantom{12}
\vspace{-.5cm}
}

\newcommand{\HO}[1]{\textcolor{orange}{#1}}

\usepackage{xfrac}

\usepackage{placeins}
\begin{document}

\maketitle

\begin{abstract}
Uncertainty quantification in Large Language Models (LLMs) is crucial for applications where safety and reliability are important. In particular, uncertainty can be used to improve the trustworthiness of LLMs by detecting factually incorrect model responses, commonly called hallucinations. 
Critically, one should seek to capture the model's \emph{semantic uncertainty}, i.e., the uncertainty over the \emph{meanings} of LLM outputs, rather than uncertainty over lexical or syntactic variations that do not affect answer correctness.
To address this problem, we propose \emph{Kernel Language Entropy} (KLE), a novel method for uncertainty estimation in white- and black-box LLMs.
KLE defines positive semidefinite unit trace kernels to encode the \emph{semantic similarities} of LLM outputs and quantifies uncertainty using the von Neumann entropy.
It considers pairwise semantic dependencies between answers (or semantic clusters), providing more fine-grained uncertainty estimates than previous methods based on hard clustering of answers.
We theoretically prove that KLE generalizes the previous state-of-the-art method called semantic entropy and empirically demonstrate that it improves uncertainty quantification performance across multiple natural language generation datasets and LLM architectures.

\end{abstract}

\begin{figure}[htbp!]
     \centering
     \includegraphics[width=\textwidth]{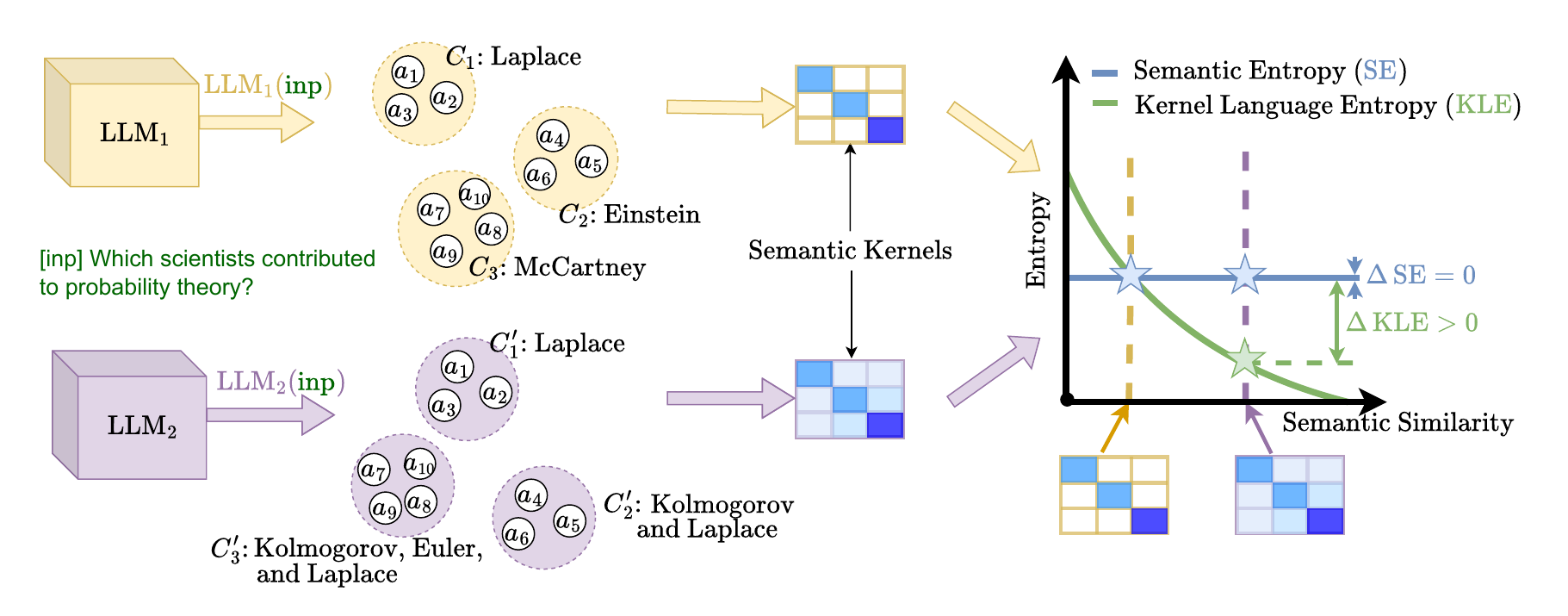}
     \vspace{-7mm}
     \caption{%
     \looseness=-1
     Illustration of Kernel Language Entropy (KLE). We here show a version of KLE called $\operatorname{KLE-c}$, which operates on semantic clusters.
     Given an input query and two different LLMs, we sample 10 answers from each model $a_1, \ldots, a_{10}$ and $a^{\prime}_1, \ldots, a^{\prime}_{10}$ and cluster them by semantic equivalence into clusters $C_1, \ldots, C_3$ and $C^{\prime}_1, \ldots, C^{\prime}_3$.
     For the sake of the example, we assume that the numbers and sizes of clusters, as well as individual cluster probabilities, are all equal $p(C_i|\text{inp}) = p(C_i^{\prime} | \text{inp})$ for all $i$.
     Then, semantic entropy would yield identical uncertainties for both LLMs.
     However, uncertainty should be lower for $\operatorname{LLM_2}$ because semantic ``similarity'' between the generations is much higher; i.e., the model is fairly confident that ``Kolmogorov'' and ``Laplace'' are good answers.
     KLE, explicitly accounts for the semantic similarity between texts using a kernel-based approach and correctly identifies that $\operatorname{LLM_2}$'s generations should be assigned lower uncertainty (see right).
}
     \label{fig:method_overview}
     \vspace{-15pt}
\end{figure}

\vspace{-15pt}
\section{Introduction}
\label{section:introduction}
\vspace{-2mm}

\looseness=-1
Large Language Models (LLMs) have demonstrated exceptional capabilities across a wide array of natural language processing tasks \citep{achiam2023gpt, team2023gemini, touvron2023llama}. This has led to their application in many domains, including medicine \citep{clusmann2023future}, education \citep{kasneci2023chatgpt}, and software development \citep{le2022coderl}.
Unfortunately, LLM generations suffer from so-called hallucinations, commonly defined as responses that are ``nonsensical or unfaithful to the provided source content'' \citep{ji2023survey, filippova2020controlled, maynez2020faithfulness}.
Hallucinations pose significant risks when LLMs are deployed to high-stakes applications, and methods that reliably detect them are sorely needed.
% increase LLM reliability are sorely needed.

\looseness=-1
A promising direction to improve the reliability of LLMs is \emph{estimating the uncertainty} of model generations \citep{kuhn2023semantic,cole2023selectively,manakul2023selfcheckgpt,lin2023generating, hendrycks2021unsolved}.
As LLM predictions tend to be well-calibrated \citep{achiam2023gpt,kadavath2022language}, high predictive uncertainty is indicative of model errors or hallucinations in settings such as answering multiple-choice questions.
This allows us to prevent harmful outcomes by abstaining from prediction or by consulting human experts.
% depending on the application. 
% However, how to best estimate uncertainty for free-form natural language generation remains an active research question.
However, the best means of estimating uncertainty for free-form natural language generation remains an active research question.
The unique properties of LLMs and natural language preclude the application of established methods for uncertainty quantification~\citep{gal2016dropout, lakshminarayanan2017simple, liu2020simple, ovadia2019can, mukhoti2023deep}.

A particular challenge is that language outputs can contain multiple types of uncertainty, including lexical (which word is used), syntactic (how the words are ordered), and semantic (what a text means). For many problems, \emph{semantic} uncertainty is the desired quantity, as it pertains directly to the accuracy of the meaning of the generated response. However, measuring the uncertainty of the generation via token likelihoods conflates all types of uncertainty.
% and does not fully represent semantic uncertainty of a text. 
% Semantic entropy (SE) \citep{kuhn2023semantic} was introduced to address this problem. 
To address this, \citet{kuhn2023semantic} have recently introduced semantic entropy (SE), which estimates uncertainty as the predictive entropy of generated texts with respect to clusters of identical semantic meaning (we discuss this in more detail in \cref{section:background}). A critical limitation of SE is that it captures semantic relations between the generated texts only through equivalence relations.
This does not capture a \emph{distance metric} in the semantic space, which would allow one to account for more nuanced \emph{semantic similarity} between generations.
For instance, it separates ``apple'' as equally strongly from ``house'' as it will ``apple'' from ``granny smith'' even though the latter pair is more closely related. In this paper, we address this problem by incorporating a distance in the semantic space of generated answers into the uncertainty estimation.

% \textbf{Kernel Language Entropy.} 
We propose \textbf{Kernel Language Entropy (KLE)}.
KLE leverages semantic similarities by using a distance measure in the space of the generated answers, encoded by unit trace positive semidefinite kernels. 
We quantify uncertainty by measuring the von Neumann entropy of these kernels.
This approach allows us to incorporate a metric between generated answers or, alternatively, semantic clusters into the uncertainty estimation.
Our approach uses kernels to describe semantic spaces, making KLE more general and better at capturing the semantics of generated texts than the previous methods. We theoretically prove that our method is more expressive than semantic entropy, meaning there are cases where KLE, but not SE, can distinguish the uncertainty of generations. Importantly, our approach does not rely on token likelihood and works for both white-box and black-box LLMs.

% \textbf{Contributions.} 
% Our work makes progress towards better uncertainty quantification in LLMs. In this paper, we 
Our work makes the following contributions towards better uncertainty quantification in LLMs:
\begin{itemize}[topsep=0pt,itemsep=0pt]
    \item We propose Kernel Language Entropy, a novel method for uncertainty quantification in natural language generation (\cref{section:theory}),
    \item We propose concrete design choices for our method that are effective in practice, for instance, graph kernels and weight functions (\cref{section:semantic_graph_kernels}),
    \item We prove that our method is a generalization of semantic entropy (\cref{proposition:texts_kernel}),
    %\item show a connection of LLM generations analysis to quantum information theory that can open avenues for novel research (\cref{section:theory}),
    \item We empirically compare our approach against baselines methods across several tasks and LLMs with up to 70B parameters (60 scenarios total), achieving SoTA results (\cref{section:experiments}).
    % and achieve SoTA results.
\end{itemize}

%TODO: add URL
We release the code and instructions for reproducing our results at \url{https://github.com/AlexanderVNikitin/kernel-language-entropy}. 

\vspace{-10pt}
\section{Background}
\label{section:background}
\vspace{-2mm}

\textbf{Uncertainty Estimation.} 
Information theory \citep{mackay2003information} offers a principled framework for quantifying the uncertainty of predictions as the predictive entropy of the output distribution:
\begin{equation}
    \smash{\operatorname{PE}(x)=H(Y \mid x)=-\int p(y \mid x) \log p(y \mid x) d y,} 
    \label{eq:predictive_entropy}
\end{equation}
where $Y$ is the output random variable, $x$ is the input, and $H(Y | x)$ is a conditional entropy which represents average uncertainty about $Y$ when $x$ is given.
Uncertainty is often categorized into aleatoric (data) and epistemic (knowledge) uncertainty.
Following previous work on uncertainty quantification in LLMs, we assume that LLMs capture both types of uncertainty \citep{kadavath2022language} and do not attempt to disambiguate them, as both epistemic and aleatoric uncertainty contribute to model errors.

\textbf{UQ in sequential models.} Let $S \in \mathcal{T}^N$ be a sequence of length $N$, consisting of tokens, $s_i \in \mathcal{T}$, where the set $\mathcal{T}$ denotes a vocabulary of tokens.
The probability of $S$ is then the joint probability of the tokens, obtained as the product of conditional token probabilities:
\begin{equation}
    p(S \mid x) = \prod\nolimits_{i}{p(s_i | s_{<i}, x)}.
    \label{eq:sequence_likelihood}
\end{equation}
Instead of \cref{eq:sequence_likelihood}, the geometric mean of token probabilities has proven to be successful in practice \citep{malinin2020uncertainty}. Using \cref{eq:predictive_entropy} and~\eqref{eq:sequence_likelihood}, we can define the predictive entropy of a sequential model.
\begin{definition}
    The predictive entropy for a random output sequence $S$ and input $x$ is
    \begin{equation}
        U(x)=H(S \mid x)=-\sum\nolimits_{s} p(s \mid x) \log (p(s \mid x)),
    \end{equation}
    where the sum is taken over all possible output sequences $s$.
\end{definition}

\looseness=-1
A downside of naive predictive entropy for Natural Language Generation (NLG) is that it measures uncertainty in the space of tokens while the uncertainty of interest lies in semantic space. As an illustrative example, consider two sets of $n$ answers, $S_i$ and $S^{\prime}_i$ sampled from two LLMs with equivalent token likelihood $p(S_i|x) = p(S^{\prime}_i|x)$ as a response to the question ``What is the capital of France?'' \citep{kuhn2023semantic}. Suppose the answers from the first LLM are various random cities (``Paris'', ``Rome'', etc.), and those from the second LLM are paraphrases of the correct answer ``It is Paris''. Naive predictive entropy computation can give similar values, even though the second LLM is not uncertain about the meaning of its answer.
% In order to resolve this problem, semantic entropy was proposed by \citep{kuhn2023semantic}.
\citet{kuhn2023semantic} have proposed semantic entropy to address this problem.

We first define the concept of semantic clustering.
Semantic clusters are equivalence classes obtained using a semantic equivalence relation, $E(\cdot, \cdot)$, which is reflexive, symmetric, and transitive and should capture semantic equivalence between input texts.
In practice, $E$ is computed using bi-directional entailment predictions from a Natural Language Inference (NLI) model, such as DeBERTa \citep{he2020deberta} or a prompted LLM, that classifies relations between pairs of texts as ``entailment,'' ``neutral,'' or ``contradiction''. 
Two texts are semantically equivalent if they entail each other bi-directionally.
Semantic clusters are obtained by greedily aggregating generations into clusters of equivalent meaning.
We can now define semantic entropy.
\begin{definition}
    \label{def:semantic_entropy}
    For an input $x$ and semantic clusters $C \in \Omega$, where $\Omega$ is a set of all semantic clusters, Semantic Entropy ($\operatorname{SE}$) is defined as 
    \begin{equation}
        \operatorname{SE}(x)=-\sum_{C \in \Omega} p(C \mid x) \log p(C \mid x)=-\sum_{C \in \Omega}\left(\left(\sum_{s \in c} p(s \mid x)\right) \log \left[\sum_{s \in C} p(s \mid x)\right]\right).
        % \operatorname{SE}(x)=-\sum_{C \in \Omega} p(C \mid x) \log p(C \mid x)=-\sum_{C \in \Omega}\left(\left(\sum_{s \in c} p(s \mid x)\right) \log \left[\sum_{s \in C} p(s \mid x)\right]\right).
    \end{equation}
\end{definition}
In practice, it is not possible to calculate $\sum_C p(C \mid x) \log p(C \mid x)$ because of the intractable number of semantic clusters. Instead, SE uses a Rao-Blackwellized Monte Carlo estimator
\begin{equation}
    \label{eq:mc_se}
        \operatorname{SE}(x) \approx - \sum\nolimits_{i=1}^{M}{p^{\prime}(C_i | x)\log{p^{\prime}(C_i | x)}},
\end{equation}
where $C_i$ are $M$ clusters extracted from the generations and $p^{\prime}(C_i \mid x)$ is a normalized semantic probability, $p^{\prime}(C_i \mid x) = \sfrac{p(C_i | x)}{\sum_i{p(C_i | x)}}$, which we refer to as $p(C_i | x)$ in the following for simplicity.
% Henceforth, we will refer to the normalized semantic probabilities simply as $p(C_i | x)$.
SE can be extended to cases where token likelihoods are not available by approximating $p(C_i|x)$ with the fraction of generated texts in each cluster, $p(C_i|x) \approx \sum_{i=1}^{N}{\sfrac{\mathbb{I}(S_i \in C_i)}{N}}$.
We refer to this variant as \emph{Discrete Semantic Entropy}~\citep{nature_paper}.

\vspace{-8pt}
\section{Kernel Language Entropy}
\label{section:theory}
\vspace{-2mm}

\looseness=-1
This section introduces Kernel Language Entropy (KLE), our novel approach to computing semantic uncertainty that accounts for fine-grained similarities between generations for better uncertainty quantification. We introduce two variants of KLE: the first, simply called $\operatorname{KLE}$, operates directly on the generated texts, and the second, $\operatorname{KLE-c}$ operates on the space of semantic clusters.

\textbf{Motivating Example.} 
\Cref{fig:method_overview} illustrates the advantages of KLE (to be precise, the $\operatorname{KLE-c}$ variant) over other methods such as SE.
% Imagine querying two LLMs with the same question: $\text{LLM}_1$ outputs disconnected semantic clusters, whereas $\text{LLM}_2$ outputs semantic clusters which are semantically related. 
Imagine querying two LLMs such that the outputs of $\text{LLM}_1$ are all semantically different and those of $\text{LLM}_2$ are semantically similar \emph{but not equivalent}. 
For simplicity, we assume an equal amount of clusters between LLMs and equal likelihoods of clusters $p(C_i | \text{inp}) = p(C_i^{\prime}|\text{inp})$.
SE would not distinguish between those cases and, thus, would misleadingly predict equal uncertainty.
KLE on the other hand, will correctly assign lower uncertainty to the outputs of $\operatorname{LLM}_2$, its kernels accounting for the fact that $\operatorname{LLM}_2$ produces semantically similar outputs.

% will consider the kernel in the semantic space and recognize that $\operatorname{LLM}_2$ produces semantically more connected outputs. 
% Consequently, it will assign lower uncertainty to the outputs of $\operatorname{LLM}_2$ based on the kernel. 

Before introducing KLE, we recall the definition of a positive semidefinite (PSD) kernel.
\begin{definition}
    \label{definition:kernel}
    For a set $\mathcal{X} \neq \emptyset$, a symmetric function $K: \mathcal{X} \times \mathcal{X} \rightarrow \mathrm{R}$ is called a PSD kernel if for all $n>0, x_i \in \mathcal{X}, \alpha_i \in \mathbb{R}$
    \begin{equation}
        \sum\nolimits_{i=1}^n\sum\nolimits_{j=1}^n \alpha_i \alpha_j K(x_i, x_j) \geq 0.
    \end{equation}
    For a finite set $\mathcal{X}$, a PSD kernel is a PSD matrix of the size $|\mathcal{X}|$.
\end{definition}
Next, we define \textbf{\emph{semantic kernels}}, denoted $K_\text{sem}$, as unit trace\footnote{Kernels with $\Tr[K] = 1$ are called \emph{unit trace kernels}.} positive semidefinite kernels over the finite domain of \emph{generated} texts. Unit trace PSD matrices are also called density matrices.
These kernels should, informally speaking, capture the semantic similarity\footnote{Or more broadly semantic \emph{relatedness}, including antonymy, meronymy, as well as semantic similarity \citep{budanitsky2006evaluating}.} between the texts such that $K(s_1, t_1) > K(s_2, t_2)$ if and only if texts $s_1$ and $t_1$ are more semantically related than texts $s_2$ and $t_2$. 
Analogously, we define semantic kernels over semantic clusters of texts, in which case the kernel should capture the semantic similarity between the clusters. In practice there are multiple ways to concretely specify a proper semantic kernel, and some options are described in Section \ref{section:semantic_graph_kernels}. %It is worth noting that, in general, such kernels can be a function of a provided input.

\paragraph{The von Neumann Entropy.} 
We propose to use the von Neumann entropy (VNE) to evaluate the uncertainty associated with a semantic kernel.
\begin{definition}[Von Neumann Entropy]
    For a unit trace positive semidefinite matrix $A \in \mathbb{R}^{n \times n}$, the von Neumann entropy (VNE) \citep{von2018mathematical} is defined as 
    \begin{equation}
        \VNE(A) = -\Tr[A \log{A}].
    \end{equation}
\end{definition}
It can be shown that $\VNE(A) = \sum_i^n - \lambda_i \log \lambda_i$ where $\lambda_i, 1\leq i \leq n$ are the eigenvalues of $A$. Within this definition, we assume $0 \log 0 = 0$. This reformulation shows that VNE is, in fact, the Shannon entropy over the eigenvalues of a kernel. We can now define Kernel Language Entropy.
\begin{definition}[Kernel Language Entropy]
    Given a set of LLM generations $S_1, \ldots, S_N$, an input $x$, and semantic kernel $K_{\text{sem}}$ over these generations and input, we define \textbf{Kernel Language Entropy} ($\operatorname{KLE}$) as the von Neumann entropy of a semantic kernel $K_\text{sem}$: 
    \begin{equation}
        \label{eq:KLE}
        \operatorname{KLE}(x) = \VNE(K_\text{sem}).
    \end{equation}
\end{definition}
The von Neumann entropy has the following properties, which is aligned with the overarching goal of measuring the uncertainty of a set of generations.
\begin{FrameProposition}[Properties of the von Neumann Entropy \citep{bengtsson2017geometry}]
    The VNE of a unit trace positive semidefinite kernel has the following properties:
    \begin{enumerate}[topsep=0pt,itemsep=0pt]
        \item The VNE of a kernel with only one non-zero element is equal to 0.
        \item The VNE is invariant under changes of basis $U$: $\VNE(K) = \VNE(U K U^{\transposed})$.
        \item The VNE is concave. For a set of positive coefficients $\alpha_i$, $\sum_{i=1}^k \alpha_i = 1$, and density matrices $K_i$, it holds that $\VNE\left(\sum\nolimits_{i=1}^{k} \alpha_i K_i\right) \geq \sum\nolimits_{i=1}^{k}{\alpha_i \VNE(K_i)}$.
    \end{enumerate}
\end{FrameProposition}

Let us briefly discuss the practical implications of these properties. \textbf{Property 1} states that if an LLM outputs a single answer (for $\operatorname{KLE}$) or a semantic cluster (for $\operatorname{KLE-c}$), the VNE is zero, indicating high certainty. \textbf{Property 2} is significant as it allows the VNE to be calculated in practice as the Shannon entropy of the diagonal elements of an orthogonalized kernel, which can be interpreted as a disentangled representation of a semantic kernel. 
\textbf{Property 3} states that entropy is concave, meaning that the entropy of a combined system is greater than or equal to the entropy of its individual parts, a common requirement for entropy metrics. 
The intuition behind our use of VNE for LLMs also relates to its origins in quantum information.

\begin{wrapfigure}{R}{0.5\textwidth}
\vspace{-12pt}
% should be careful not to go above header unfortunately
\begin{minipage}{0.5\textwidth}
\begin{algorithm}[H]
    \caption{Kernel Language Entropy}
    \label{listing:kle}
    \begin{algorithmic}[1]
        \Require $\operatorname{LLM}$, Input $x \in \mathcal{T}^L$, Number of samples $n$, Boolean $\operatorname{kle-c}$ indicating variant, Semantic kernels $K_i$
        \State Initialize a \emph{multiset} of answers $\mathcal{O} \gets \emptyset$
        \For{$k \gets 1$ to $n$} \algorithmiccomment{Sampling $n$ answers}
            \State Add $\operatorname{LLM}(x)$ to $\mathcal{O}$ 
        \EndFor
        \If{$\operatorname{kle-c}$}
            \State Update $\mathcal{O} \gets \operatorname{cluster(\mathcal{O})}$ \algorithmiccomment{as in \citep{kuhn2023semantic}}
        \EndIf
        \State Combine $K_i(\mathcal{O}, \mathcal{O})$ in $K_\text{sem}$ \algorithmiccomment{see \cref{section:combination}}
        \State Return $\VNE(K_\text{sem})$ \algorithmiccomment{\cref{eq:KLE}}
    \end{algorithmic}
\end{algorithm}
\end{minipage}
\vspace{-10pt}
\end{wrapfigure}

\looseness=-1
\paragraph{The VNE in Quantum Information Theory.} 
In quantum information theory, the states of a quantum system (or pure states) are defined as unit vectors in $\mathbb{C}^{N}$. However, experiments often result in statistical mixtures of pure quantum states, represented as density matrices. 
The VNE is used to evaluate the entropy of the mixed states.
Analogously, we can think of KLE as considering each answer as a mixture of pure ``semantic meanings'', measuring the entropy of this mixture. 
We refer the reader to \citet{aaronson2022introduction} for further background reading on the VNE and quantum information theory.

\paragraph{KLE-c.}\label{section:kle-c}
Instead of defining semantic kernels directly over individual model generations, we can also apply KLE to clusters of semantic equivalence.
We call this variant of our method $\operatorname{KLE-c}$.
Although $\operatorname{KLE}$ is more general than $\operatorname{KLE-c}$ for non-trivial clusterings, KLE-c can provide practical value as it is cheaper to compute and more interpretable due to its smaller kernel sizes.

\textbf{Algorithm.} \Cref{listing:kle} provides a generic description of the steps required to compute KLE.
We describe the practical details for defining and combining semantic kernels later in \cref{section:combination}.

\textbf{Computational Complexity}. The computational complexity of KLE is approximately identical to SE which requires sampling from an LLM $N$ times and running the entailment model $O(N^2)$ times.
Additionally, KLE requires $O(N^3)$ elementary operations for kernel and VNE calculation.
The actual cost of this is negligible in comparison to the forward passes through the LLM or entailment model.

\vspace{-1.5mm}
\subsection{Semantic Graph Kernels}
\label{section:semantic_graph_kernels}
% 
% \begin{figure}[t]
%      \centering
%      \includegraphics[width=\textwidth]{images/convergence_plots.pdf}
%      \vspace{-2mm}
%      \caption{Entropy Convergence Plots for heat kernels. For graphs of various sizes $|V|$, we grow the number of edges and examine the VNE.
%      % of a target kernel. 
%      For large lengthscales, the VNE quickly converges to zero. We can use these plots to determine kernel hyperparameters without validation sets.
%      % these plots to determine kernel hyperparameters
%      % These plots should be considered when determining kernel hyperparameters without a validation set.
%      }
%      \label{fig:convergence_plots}
%      \vspace{-10pt}
% \end{figure}
% 
This section describes a practical approach for constructing semantic kernels over LLM generations or semantic clusters. Concretely, we apply NLI models to construct \textit{semantic graphs} over the LLM outputs and then borrow from graph kernel theory to construct kernels from these graphs.

\paragraph{Graph Theory Preliminaries.} First, let us recall the basics of graph theory. A graph is a pair of two sets $G = (V, E)$, where $V=\{1, \ldots, n\}$ is a set of $n$ vertices and $E \subseteq V \times V$ is a set of edges. A graph is called weighted when a weight is assigned to each edge, and the weight matrix $W_{ij}$ contains weights between nodes $i$ and $j$. For unweighted graphs, we can use a binary adjacency matrix to encode edges between nodes. The degree matrix $D$ is a diagonal $|V| \times |V|$ matrix with $\smash{D_{ii} = \sum\nolimits_{j=1}^{|V|} W_{ij}}$.
% The adjacency matrix analogously is a binary matrix used to encode graphs without weights. 
The \textit{graph Laplacian} is defined as $L = D - W$. $L$ is a positive semidefinite matrix, and eigenvalues of $L$ are often used to study the structure of graphs \citep{chung1997spectral, von2007tutorial}.

\paragraph{Semantic Graph.} We define semantic graphs as graphs over LLM generations ($G_{\text{sem}}$) or semantic clusters ($G_{\text{sem-c}}$). For $G_{\text{sem}}$, edges can be defined as a function of NLI predictions in both directions: $W_{ij} = f(\operatorname{NLI}(S_i, S_j), \operatorname{NLI}(S_j, S_i))$, where $\operatorname{NLI}$ are the predicted probabilities for \textit{entailment}, \textit{neutral}, and \textit{contradiction} for $S_i$ and $S_j$. For example, $f$ could be the weighted sum over the predicted probabilities for entailment and neutral classes.
For $G_{\text{semc-c}}$, the weights between the clusters are computed by summing the entailment predictions over the generations assigned to the clusters,  $W_{ij} = \sum_{s \in C_i} \sum_{t \in C_j} f(\operatorname{NLI}(s, t), \operatorname{NLI}(t, s))$.

\begin{figure}[t]
     \centering
     \includegraphics[width=\textwidth]{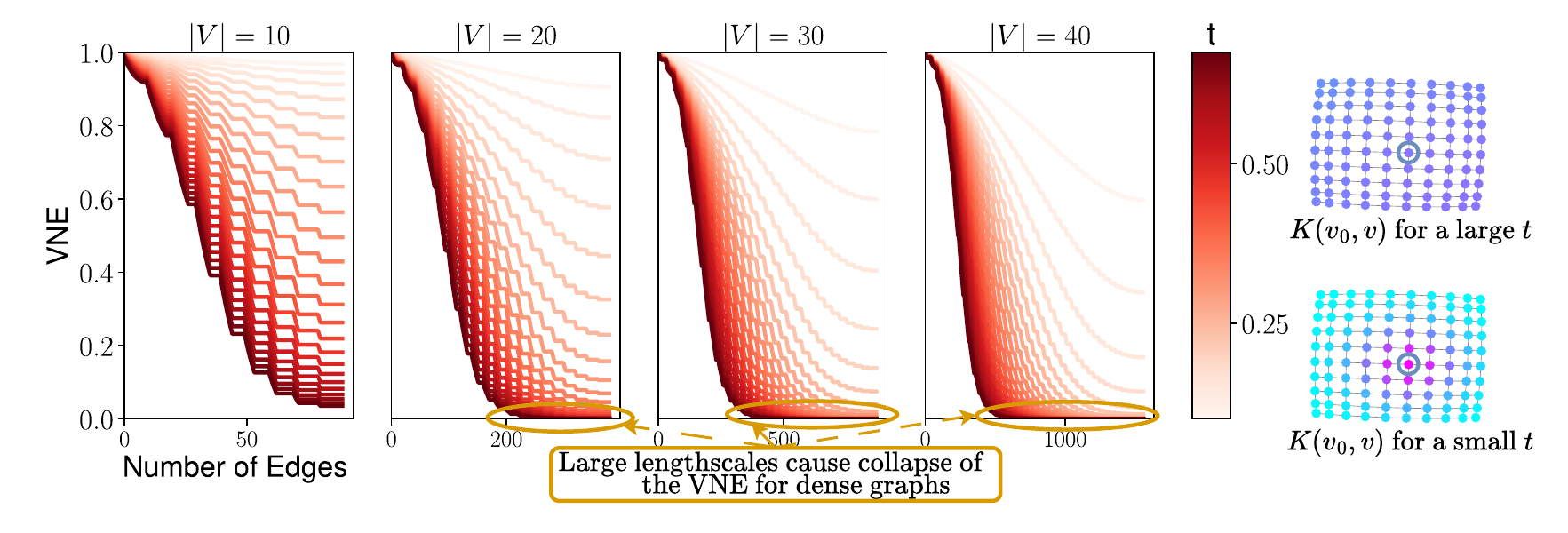}
     \vspace{-6mm}
     \caption{Entropy Convergence Plots for heat kernels. For graphs of various sizes $|V|$, we grow the number of edges and examine the VNE.
     For large lengthscales $t$, corresponding to darker colored curves, the VNE quickly converges to zero. We can use these plots to determine kernel hyperparameters without validation sets. The VNE is scaled to start at 1 for visualization purposes.
     }
     \label{fig:convergence_plots}
     \vspace{-10pt}
\end{figure}

\paragraph{Graph Kernels.} 
When a semantic graph is obtained, KLE calculates graph kernels over semantic graph nodes to compute a distance measure. 
Since graphs are discrete and finite, any positive semidefinite matrix would be a kernel over the graph. 
% ¨¨
However, we seek kernels that exploit knowledge about the graph structure.
We therefore adopt Partial Differential Equation (PDE) and Stochastic Partial Differential Equation (SPDE) approaches to graph kernels \citep{kondor2002diffusion, borovitskiy2021matern, nikitin2022non}. 
If $u \in \R^n$ is a signal over the nodes of a graph, the \textbf{heat kernel} is a solution to the partial differential equation $\sfrac{\partial u}{\partial t} + Lu = 0$ and the \textbf{Mat\'ern kernel} is a solution to the stochastic differential equation, $(\sfrac{2 \nu}{\kappa^2}+L)^{\frac{\nu}{2}} u=w$, where $w$ is white noise over the graph nodes and $L$ is the graph Laplacian defined above. The corresponding solutions to these equations are:
\begin{equation}
    K_t = e^{-tL} \quad \textsc{[heat]} \quad\quad\quad\quad\quad\quad\quad\quad\quad K_{\nu \kappa }=(\sfrac{2 \nu}{\kappa^2} I+L)^{-\nu} \quad \textsc{[mat\'ern]}.
\end{equation}
These kernels allow for the incorporation of a distance measure that reflects the graph's locality properties (right part of \cref{fig:convergence_plots}). For example, the Taylor series of the heat kernel can be shown to be equal to a sum of powers of random walk matrices. Both kernels have hyperparameters: lengthscales $t$ in the heat kernel and $\kappa$ in Mat\'ern kernels, and $\nu$ in the Mat\'ern kernel, often interpreted as smoothness.
The scaled eigenvalues of the Mat\'ern kernel converge to the eigenvalues of the heat kernel \citep{borovitskiy2021matern} when $\nu$ goes to infinity. 
Mat\'ern kernels provide more flexibility at the cost of the additional parameter. 
% Note that any kernel can be normalized into a unit trace kernel via $\smash{K(x, y)\leftarrow \frac{K(x, y)}{N \sqrt{K(x, x) K(y, y)}}}$, where $N$ is the size of $K$.
Note that any kernel can be normalized into a unit trace kernel via $\smash{K(x, y)\leftarrow K(x, y) (K(x, x) K(y, y))^{-1/2}/N}$, where $N$ is the size of $K$.
We refer to \citep{kondor2002diffusion, nikitin2022non, borovitskiy2021matern} for further background reading on graph kernels.

\paragraph{Kernel Hyperparameters.} 
We propose two ways to select the hyperparameters of the heat and Mat\'ern kernels: either by maximizing the validation set performance or by selecting parameters from what we call \emph{Entropy Convergence Plots}, illustrated in \cref{fig:convergence_plots}.
We obtain these plots by defining a set of progressively denser graphs $G_1 \prec \ldots \prec G_K$. These can be obtained by starting from a graph without edges and a fixed number of vertices and adding new edges either randomly or by filling in the adjacencies of each node sequentially. We then plot the VNE against the number of edges in the graphs $G_i$.
We analyze the von Neumann entropy over these plots to avoid pathologies connected to the fact that for large lengthscales, the VNE converges rather quickly and such behavior should generally be avoided. 
For all remaining values, we can either choose hyperparameters randomly from the range of non-collapsing hyperparameters or rely on prior domain knowledge.
% a domain expert can choose the hyperparameters either from prior knowledge about the problem or randomly from the range of non-collapsing hyperparameters. 

\paragraph{Kernel Combination.}\label{section:combination} 
KLE offers the additional flexibility of combining kernels from various methods (e.g., multiple NLI models, different graph kernels, or other methods). 
For example, we can combine multiple kernels using convex combinations, $\smash{K = \sum_{i=1}^P \alpha_i K_i}$, where $\smash{\sum_{i=1}^{P}{\alpha_i} = 1}$.
% , or by taking Kronecker product and normalizing the result.

\vspace{-1mm}
\subsection{Kernel Language Entropy Generalizes Semantic Entropy}
\vspace{-1mm}
The semantic kernels used in KLE are more informative than the semantic equivalence relations used in SE \citep{kuhn2023semantic}. The next theorem shows that KLE can recover SE for any semantic clustering.
\begin{FrameTheorem}[KLE and KLE-c generalize SE]
    \label{proposition:texts_kernel}
    For any semantic clustering, there exists a semantic kernel over texts $K_{\text{sem}}(s, s^{\prime})$ such that the VNE of this kernel is equal to semantic entropy (computed as in \cref{eq:mc_se}). Moreover, there exists a semantic kernel over clusters $K_{\text{sem}}(c, c^{\prime})$ such that the VNE of this kernel is equal to SE.
\end{FrameTheorem}
\begin{proof}[Proof Sketch]
For any semantic clustering, we consider a kernel with a block diagonal structure. Each block corresponds to a semantic cluster and cluster likelihoods are normalized by the size of the cluster, $\sfrac{p(C_i | x)}{m_i}$. This is a valid semantic kernel and the KLE for this kernel equals the SE. \cref{appendix:proposition:texts_kernel} and \cref{appendix:proposition:sem_clusters_kernel} in the Appendix contain the detailed proofs.
\end{proof}
The proof of \cref{proposition:texts_kernel} shows that the block diagonal semantic kernels used with KLE can recover semantic entropy for any clustering. However, there are other kernels available that allow KLE to be more expressive than SE.
% even SE with varying levels of granularity of semantic clustering. 
%In the next theorem, we show that KLE-c is also more general than semantic entropy.
% 
% 
%\begin{FrameTheorem}[KLE-c is more general than SE]
%        \label{proposition:sem_clusters_kernel}
%    For any semantic clustering, there exists a kernel over semantic clusters $K_{\text{sem}}(c, c^{\prime})$ such that the VNE of this kernel is equal to semantic entropy (computed as in Eq.~\eqref{eq:mc_se}).
%\end{FrameTheorem}
% 
Comparing KLE and KLE-c, we find that KLE is more general than KLE-c for any non-trivial clustering.

\vspace{-2mm}
\section{Related Work} 
\label{section:related_work}
% \vspace{-2mm}

In the context of machine learning, the VNE has been studied theoretically\citep{bach2022information}, applied to GAN regularization \citep{kim2023vne}, and the exponential of the VNE has been used for effective rank and sample diversity analysis \citep{roy2007effective,friedman2023vendi}.

The first attempts at estimating the entropy of language date back to the 1950s \citep{shannon1951prediction}, and, today, techniques for uncertainty quantification are widely used in natural language processing. For instance, \citet{desai2020calibration} and \citet{jiang2021can} presented calibration techniques for classification tasks. \citet{xiao2019quantifying} empirically showed that, for various tasks including sentiment analysis and named entity recognition, measuring model uncertainty can be used to improve performance.
% helps enhance performance and confidence. 
Calibration techniques have also been applied in machine translation tasks to improve accuracy \citep{kumar2019calibration}.

\citet{malinin2020uncertainty} discussed the challenges of estimating uncertainty in sequential models.
% and propose predictive entropy, which later was used in many approaches, for instance in \citep{xiao2021hallucination}.
Several previous works have queried LLMs to elicit statements about uncertainty, either via fine-tuning or by directly including previous LLM generations in the prompt \citep{kadavath2022language, chen2023quantifying, mielke2022reducing, lin2022teaching, mielke2020linguistic, ganguli2023capacity, ren2023self, tian2023just, cohen2023lm,xiao2021hallucination,kuhn2023semantic}. \citet{zhang2024luq} studied UQ for long text generation. \citet{quach2023conformal} used conformal predictions to quantify LLM uncertainty, which is orthogonal to the approach we pursue here.
\citet{yang2023bayesian} have shown that Bayesian modeling of LLMs using low-rank Laplace approximations improves calibration in small-scale multiple-choice settings. \citet{lin2023generating} applied spectral graph analysis to graphs of answers for black-box LLMs.
\citet{aichberger2024many} proposed a new method for sampling diverse answers from LLMs; more diverse sampling strategies could improve KLE as well. 
% Lastly, \citet{kuhn2023semantic} proposed semantic entropy, which calculates predictive entropy over semantic equivalence classes and importantly. Our contribution is in utilizing semantic kernels to estimate the uncertainty of NLG. Our approach can flexibly incorporate semantic similarities, and efficiently captures semantic uncertainty.

There are a variety of ways besides model uncertainty to detect hallucinations in LLMs such as querying external knowledge bases \citep{feldman2023trapping, li2023chain, varshney2023stitch}, hidden state interventions \citep{zou2023representation, hernandez2023measuring, liu2023context}, using probes \citep{burns2024discovering, li2024inference, macdiarmid2024sleeperagentprobes}, or applying fine-tuning \citep{kang2024unfamiliar, tian2023finetuning}. KLE is complementary to many of these directions and focuses on estimating more fine-grained semantic uncertainty. It can either be used to improve these approaches or be combined with them sequentially.

\vspace{-2mm}
\section{Experiments}
\label{section:experiments}
% \vspace{-2mm}

\begin{table*}[t]
\centering
\caption{Detailed experimental results for Llama 2 70B Chat and Falcon 40B Instruct.}
\label{table:exp_results_llama_falcon}
\adjustbox{max width=1\textwidth}{

\begin{tabular}{c|lllllllllll}
 & \multirow[c]{2}{*}{\textbf{Method}} & \multicolumn{2}{c}{\textbf{BioASQ} \citep{krithara2023bioasq}} & \multicolumn{2}{c}{\textbf{NQ} \citep{kwiatkowski2019natural}} & \multicolumn{2}{c}{\textbf{SQuAD} \citep{rajpurkar2018know}} & \multicolumn{2}{c}{\textbf{SVAMP} \citep{patel-etal-2021-nlp}} & \multicolumn{2}{c}{\textbf{Trivia QA} \citep{joshi2017triviaqa}} \\
 &  & \multicolumn{1}{c}{ \textbf{AUROC}} &\multicolumn{1}{c}{ \textbf{AUARC} } & \multicolumn{1}{c}{ \textbf{AUROC} } & \multicolumn{1}{c}{ \textbf{AUARC} } & \multicolumn{1}{c}{ \textbf{AUROC} } & \multicolumn{1}{c}{ \textbf{AUARC} } & \multicolumn{1}{c}{ \textbf{AUROC} } & \multicolumn{1}{c}{\textbf{AUARC}} & \multicolumn{1}{c}{ \textbf{AUROC} } & \multicolumn{1}{c}{\textbf{AUARC}} \\
\cline{1-12} \noalign{\vspace{0.5ex}} 
 \multirow[c]{7}{*}{\rotatebox{90}{Llama 2 70B Chat}} & SE \citep{kuhn2023semantic} & 0.74 \textcolor{gray}{\footnotesize $\pm$ 0.04} & 0.90 \textcolor{gray}{\footnotesize $\pm$ 0.01} & 0.71 \textcolor{gray}{\footnotesize $\pm$ 0.03} & 0.47 \textcolor{gray}{\footnotesize $\pm$ 0.03} & 0.66 \textcolor{gray}{\footnotesize $\pm$ 0.03} & 0.65 \textcolor{gray}{\footnotesize $\pm$ 0.03} & 0.62 \textcolor{gray}{\footnotesize $\pm$ 0.03} & 0.61 \textcolor{gray}{\footnotesize $\pm$ 0.03} & 0.77 \textcolor{gray}{\footnotesize $\pm$ 0.03} & 0.79 \textcolor{gray}{\footnotesize $\pm$ 0.02} \\
  & DSE \citep{kuhn2023semantic} & 0.75 \textcolor{gray}{\footnotesize $\pm$ 0.04} & 0.90 \textcolor{gray}{\footnotesize $\pm$ 0.01} & 0.71 \textcolor{gray}{\footnotesize $\pm$ 0.03} & 0.46 \textcolor{gray}{\footnotesize $\pm$ 0.03} & 0.66 \textcolor{gray}{\footnotesize $\pm$ 0.03} & 0.65 \textcolor{gray}{\footnotesize $\pm$ 0.03} & 0.63 \textcolor{gray}{\footnotesize $\pm$ 0.03} & 0.61 \textcolor{gray}{\footnotesize $\pm$ 0.03} & 0.77 \textcolor{gray}{\footnotesize $\pm$ 0.03} & 0.79 \textcolor{gray}{\footnotesize $\pm$ 0.02} \\
  & PE \citep{malinin2020uncertainty}  & 0.69 \textcolor{gray}{\footnotesize $\pm$ 0.04} & 0.90 \textcolor{gray}{\footnotesize $\pm$ 0.01} & 0.67 \textcolor{gray}{\footnotesize $\pm$ 0.03} & 0.44 \textcolor{gray}{\footnotesize $\pm$ 0.03} & 0.65 \textcolor{gray}{\footnotesize $\pm$ 0.03} & 0.65 \textcolor{gray}{\footnotesize $\pm$ 0.03} & 0.59 \textcolor{gray}{\footnotesize $\pm$ 0.03} & 0.58 \textcolor{gray}{\footnotesize $\pm$ 0.03} & 0.61 \textcolor{gray}{\footnotesize $\pm$ 0.03} & 0.73 \textcolor{gray}{\footnotesize $\pm$ 0.03} \\
  & P(True) \citep{kadavath2022language} & 0.86 \textcolor{gray}{\footnotesize $\pm$ 0.03} & \textbf{0.92} \textcolor{gray}{\footnotesize $\pm$ 0.01} & \textbf{0.78} \textcolor{gray}{\footnotesize $\pm$ 0.03} & 0.50 \textcolor{gray}{\footnotesize $\pm$ 0.03} & 0.69 \textcolor{gray}{\footnotesize $\pm$ 0.03} & \textbf{0.68} \textcolor{gray}{\footnotesize $\pm$ 0.03} & 0.74 \textcolor{gray}{\footnotesize $\pm$ 0.02} & 0.68 \textcolor{gray}{\footnotesize $\pm$ 0.03} & 0.76 \textcolor{gray}{\footnotesize $\pm$ 0.03} & 0.79 \textcolor{gray}{\footnotesize $\pm$ 0.02} \\
  & ER & 0.70 \textcolor{gray}{\footnotesize $\pm$ 0.05} & 0.89 \textcolor{gray}{\footnotesize $\pm$ 0.01} & 0.58 \textcolor{gray}{\footnotesize $\pm$ 0.03} & 0.39 \textcolor{gray}{\footnotesize $\pm$ 0.03} & 0.63 \textcolor{gray}{\footnotesize $\pm$ 0.03} & 0.64 \textcolor{gray}{\footnotesize $\pm$ 0.03} & 0.68 \textcolor{gray}{\footnotesize $\pm$ 0.03} & 0.64 \textcolor{gray}{\footnotesize $\pm$ 0.03} & 0.76 \textcolor{gray}{\footnotesize $\pm$ 0.03} & 0.79 \textcolor{gray}{\footnotesize $\pm$ 0.02} \\
  %{\color{\gray} \cline{3-12}}
  & \textbf{$\KLE(K_{\textsc{heat}})$} & 0.87 \textcolor{gray}{\footnotesize $\pm$ 0.03} & \textbf{0.92} \textcolor{gray}{\footnotesize $\pm$ 0.01} & \textbf{0.78} \textcolor{gray}{\footnotesize $\pm$ 0.02} & \textbf{0.51} \textcolor{gray}{\footnotesize $\pm$ 0.03} & \textbf{0.71} \textcolor{gray}{\footnotesize $\pm$ 0.03} & \textbf{0.68} \textcolor{gray}{\footnotesize $\pm$ 0.03} & \textbf{0.76} \textcolor{gray}{\footnotesize $\pm$ 0.02} & \textbf{0.69} \textcolor{gray}{\footnotesize $\pm$ 0.03} & \textbf{0.84} \textcolor{gray}{\footnotesize $\pm$ 0.03} & \textbf{0.82} \textcolor{gray}{\footnotesize $\pm$ 0.02} \\
 & \textbf{$\KLE(K_{\textsc{full}})$} & \textbf{0.88} \textcolor{gray}{\footnotesize $\pm$ 0.03} & \textbf{0.92} \textcolor{gray}{\footnotesize $\pm$ 0.01} & 0.77 \textcolor{gray}{\footnotesize $\pm$ 0.02} & 0.50 \textcolor{gray}{\footnotesize $\pm$ 0.03} & 0.70 \textcolor{gray}{\footnotesize $\pm$ 0.03} & \textbf{0.68} \textcolor{gray}{\footnotesize $\pm$ 0.03} & 0.70 \textcolor{gray}{\footnotesize $\pm$ 0.03} & 0.65 \textcolor{gray}{\footnotesize $\pm$ 0.03} & 0.80 \textcolor{gray}{\footnotesize $\pm$ 0.03} & 0.81 \textcolor{gray}{\footnotesize $\pm$ 0.02} \\
\cline{1-12} %Falcon 40B Instr
\noalign{\vspace{0.5ex}} 
\multirow[c]{7}{*}{\rotatebox{90}{Falcon 40B Instr}} & SE \citep{kuhn2023semantic} & 0.85 \textcolor{gray}{\footnotesize $\pm$ 0.02} & 0.90 \textcolor{gray}{\footnotesize $\pm$ 0.01} & \textbf{0.78} \textcolor{gray}{\footnotesize $\pm$ 0.03} & \textbf{0.43} \textcolor{gray}{\footnotesize $\pm$ 0.03} & 0.66 \textcolor{gray}{\footnotesize $\pm$ 0.03} & 0.63 \textcolor{gray}{\footnotesize $\pm$ 0.03} & 0.66 \textcolor{gray}{\footnotesize $\pm$ 0.03} & 0.63 \textcolor{gray}{\footnotesize $\pm$ 0.03} & 0.79 \textcolor{gray}{\footnotesize $\pm$ 0.03} & 0.72 \textcolor{gray}{\footnotesize $\pm$ 0.03} \\
 & DSE \citep{kuhn2023semantic} & 0.85 \textcolor{gray}{\footnotesize $\pm$ 0.02} & 0.89 \textcolor{gray}{\footnotesize $\pm$ 0.01} & 0.77 \textcolor{gray}{\footnotesize $\pm$ 0.03} & 0.40 \textcolor{gray}{\footnotesize $\pm$ 0.03} & 0.66 \textcolor{gray}{\footnotesize $\pm$ 0.03} & 0.62 \textcolor{gray}{\footnotesize $\pm$ 0.03} & 0.67 \textcolor{gray}{\footnotesize $\pm$ 0.03} & 0.61 \textcolor{gray}{\footnotesize $\pm$ 0.03} & 0.79 \textcolor{gray}{\footnotesize $\pm$ 0.03} & 0.71 \textcolor{gray}{\footnotesize $\pm$ 0.03} \\
 & PE \citep{malinin2020uncertainty}  & 0.75 \textcolor{gray}{\footnotesize $\pm$ 0.03} & 0.87 \textcolor{gray}{\footnotesize $\pm$ 0.01} & 0.71 \textcolor{gray}{\footnotesize $\pm$ 0.03} & 0.38 \textcolor{gray}{\footnotesize $\pm$ 0.03} & 0.63 \textcolor{gray}{\footnotesize $\pm$ 0.03} & 0.60 \textcolor{gray}{\footnotesize $\pm$ 0.03} & 0.59 \textcolor{gray}{\footnotesize $\pm$ 0.03} & 0.57 \textcolor{gray}{\footnotesize $\pm$ 0.03} & 0.68 \textcolor{gray}{\footnotesize $\pm$ 0.03} & 0.66 \textcolor{gray}{\footnotesize $\pm$ 0.03} \\
 & P(True) \citep{kadavath2022language} & 0.87 \textcolor{gray}{\footnotesize $\pm$ 0.03} & 0.89 \textcolor{gray}{\footnotesize $\pm$ 0.01} & 0.71 \textcolor{gray}{\footnotesize $\pm$ 0.03} & 0.37 \textcolor{gray}{\footnotesize $\pm$ 0.03} & 0.66 \textcolor{gray}{\footnotesize $\pm$ 0.03} & 0.61 \textcolor{gray}{\footnotesize $\pm$ 0.03} & 0.73 \textcolor{gray}{\footnotesize $\pm$ 0.03} & 0.67 \textcolor{gray}{\footnotesize $\pm$ 0.03} & 0.72 \textcolor{gray}{\footnotesize $\pm$ 0.03} & 0.69 \textcolor{gray}{\footnotesize $\pm$ 0.03} \\
 & ER & 0.74 \textcolor{gray}{\footnotesize $\pm$ 0.04} & 0.85 \textcolor{gray}{\footnotesize $\pm$ 0.02} & 0.73 \textcolor{gray}{\footnotesize $\pm$ 0.03} & 0.39 \textcolor{gray}{\footnotesize $\pm$ 0.03} & 0.63 \textcolor{gray}{\footnotesize $\pm$ 0.03} & 0.61 \textcolor{gray}{\footnotesize $\pm$ 0.03} & 0.75 \textcolor{gray}{\footnotesize $\pm$ 0.02} & \textbf{0.68} \textcolor{gray}{\footnotesize $\pm$ 0.03} & 0.76 \textcolor{gray}{\footnotesize $\pm$ 0.03} & 0.69 \textcolor{gray}{\footnotesize $\pm$ 0.03} \\
 & \textbf{$\KLE(K_{\textsc{heat}})$} & \textbf{0.92} \textcolor{gray}{\footnotesize $\pm$ 0.01} & \textbf{0.91} \textcolor{gray}{\footnotesize $\pm$ 0.01} & 0.76 \textcolor{gray}{\footnotesize $\pm$ 0.03} & 0.42 \textcolor{gray}{\footnotesize $\pm$ 0.03} & \textbf{0.70} \textcolor{gray}{\footnotesize $\pm$ 0.03} & \textbf{0.66} \textcolor{gray}{\footnotesize $\pm$ 0.03} & \textbf{0.77} \textcolor{gray}{\footnotesize $\pm$ 0.02} & \textbf{0.68} \textcolor{gray}{\footnotesize $\pm$ 0.03} & \textbf{0.80} \textcolor{gray}{\footnotesize $\pm$ 0.02} & \textbf{0.74} \textcolor{gray}{\footnotesize $\pm$ 0.03} \\
 & \textbf{$\KLE(K_{\textsc{full}})$} & 0.90 \textcolor{gray}{\footnotesize $\pm$ 0.02} & \textbf{0.91} \textcolor{gray}{\footnotesize $\pm$ 0.01} & \textbf{0.78} \textcolor{gray}{\footnotesize $\pm$ 0.03} & \textbf{0.43} \textcolor{gray}{\footnotesize $\pm$ 0.03} & 0.69 \textcolor{gray}{\footnotesize $\pm$ 0.03} & 0.65 \textcolor{gray}{\footnotesize $\pm$ 0.03} & 0.69 \textcolor{gray}{\footnotesize $\pm$ 0.03} & 0.64 \textcolor{gray}{\footnotesize $\pm$ 0.03} & \textbf{0.80} \textcolor{gray}{\footnotesize $\pm$ 0.03} & 0.73 \textcolor{gray}{\footnotesize $\pm$ 0.03} \\
%\cline{1-12}
\bottomrule
\end{tabular}}

\end{table*}

\textbf{Datasets and Models.}
Our experiments span over 60 dataset-model pairs.
We evaluate on the following tasks covering different domains of natural language generation:  general knowledge (TriviaQA \citep{joshi2017triviaqa} and SQuAD \citep{rajpurkar2018know}), biology and medicine (BioASQ \citep{krithara2023bioasq}), general domain questions from Google search (Natural Questions, NQ \citep{kwiatkowski2019natural}), and natural language math problems (SVAMP \citep{patel-etal-2021-nlp}).
We generally discard the context associated with each input for all datasets except SVAMP, as the tasks become too easy for the current generation of models when context is provided.
We use the following LLMs: Llama-2 7B, 13B, and 70B \citep{touvron2023llama}, Falcon 7B and 40B \citep{almazrouei2023falcon}, and Mistral 7B \citep{jiang2023mistral}, using both standard and instruction-tuned versions of these models.
As the NLI model for defining semantic graphs or semantic clusters, we use DeBERTa-Large-MNLI \citep{he2020deberta}.

\textbf{Baselines.} As baseline methods, we compare KLE with semantic entropy \citep{kuhn2023semantic}, discrete semantic entropy \citep{nature_paper, kuhn2023semantic}, token predictive entropy \citep{malinin2020uncertainty}, embedding regression \citep{nature_paper}, and P(True) \citep{kadavath2022language}. For embedding regression, we train a logistic regression model on the last layer hidden states to predict whether a given LLM answer is correct.

\textbf{KLE Kernels.} We propose to use the following two semantic kernels with $\operatorname{KLE}$: $K_{\textsc{heat}}$ and $K_{\textsc{full}}$. Both are obtained from the weighted graph %$W_{ij} = \langle\left[1, 0.5, 0\right]^{\transposed}, \operatorname{NLI}^{\prime}(S_i, S_j) \rangle +  \langle\left[1, 0.5, 0\right]^{\transposed}, \operatorname{NLI}^{\prime}(S_j, S_i)\rangle$
%W_{ij} = \left[1, 0.5, 0\right]^{\transposed}\operatorname{NLI}^{\prime}(S_i, S_j) +  \left[1, 0.5, 0\right]^{\transposed}\operatorname{NLI}^{\prime}(S_j, S_i)$
$W_{ij} = w\operatorname{NLI}^{\prime}(S_i, S_j) +  w\operatorname{NLI}^{\prime}(S_j, S_i)$, where $\smash{w=(1,0.5,0)^{\transposed}}$ is a weight vector. Here, we assume that $\operatorname{NLI}^{\prime}$ returns a one-hot prediction over (entailment, neutral class, contradiction). $K_{\textsc{HEAT}}$ is a heat kernel over this graph. 
We further propose $K_{\textsc{FULL}} = \alpha K_{\textsc{HEAT}} + (1 - \alpha) K_\textsc{SE}$, where $\alpha \in \left[0, 1\right]$ and $K_{\textsc{SE}}$ is a semantic entropy kernel. 
We ablate these kernel choices in our experiments below.
% We also compare this model choices to a set of other kernels below.

\textbf{Evaluation metrics.} 
% The primary usage of UQ in LLMs is for refusing to give unreliable answers. 
Following previous work, we evaluate uncertainty methods by measuring their ability to predict the correctness of model responses, calculating the Area under the Receiver Operating Curve (AUROC). 
Further, uncertainty metrics can be used to refuse answering when uncertainty is high, increasing model accuracy on the subset of questions with uncertainty below a threshold.
We measure this with the Area Under the Accuracy-Rejection Curve (AUARC), \citep{nadeem2009accuracy}.
The rejection accuracy at a given uncertainty threshold is the accuracy of the model on the subset of inputs for which uncertainty is lower than the threshold; the AUARC score computes the area under the rejection accuracy curve for all possible thresholds.

\textbf{Sampling.} 
We sample 10 answers per input via top-K sampling with $K=50$ and nucleus sampling with $p=0.9$ at temperature $T=1$.
To assess model accuracy, we draw an additional low-temperature sample ($T=0.1$) and ask an additional LLM (Llama 3 8B Instruct) to compare the model response to the ground truth answer provided by the datasets.

\textbf{Statistical significance.} We assess statistical significance in two ways. First, we run a large number of experimental scenarios (60 model-dataset pairs), and second, for each experimental scenario, we also obtain confidence intervals over $1000$ bootstrap resamples. 
We note that standard errors in each scenario are more representative of the LLM and the dataset rather than the method. Therefore, our main criterion for comparing the methods is based on the fraction of experimental cases where our method outperforms baselines (assessed with a binomial statistical significance test).

\begin{figure}[t]
\centering
\begin{subfigure}{.5\textwidth}
  \centering
  \includegraphics[width=\linewidth]{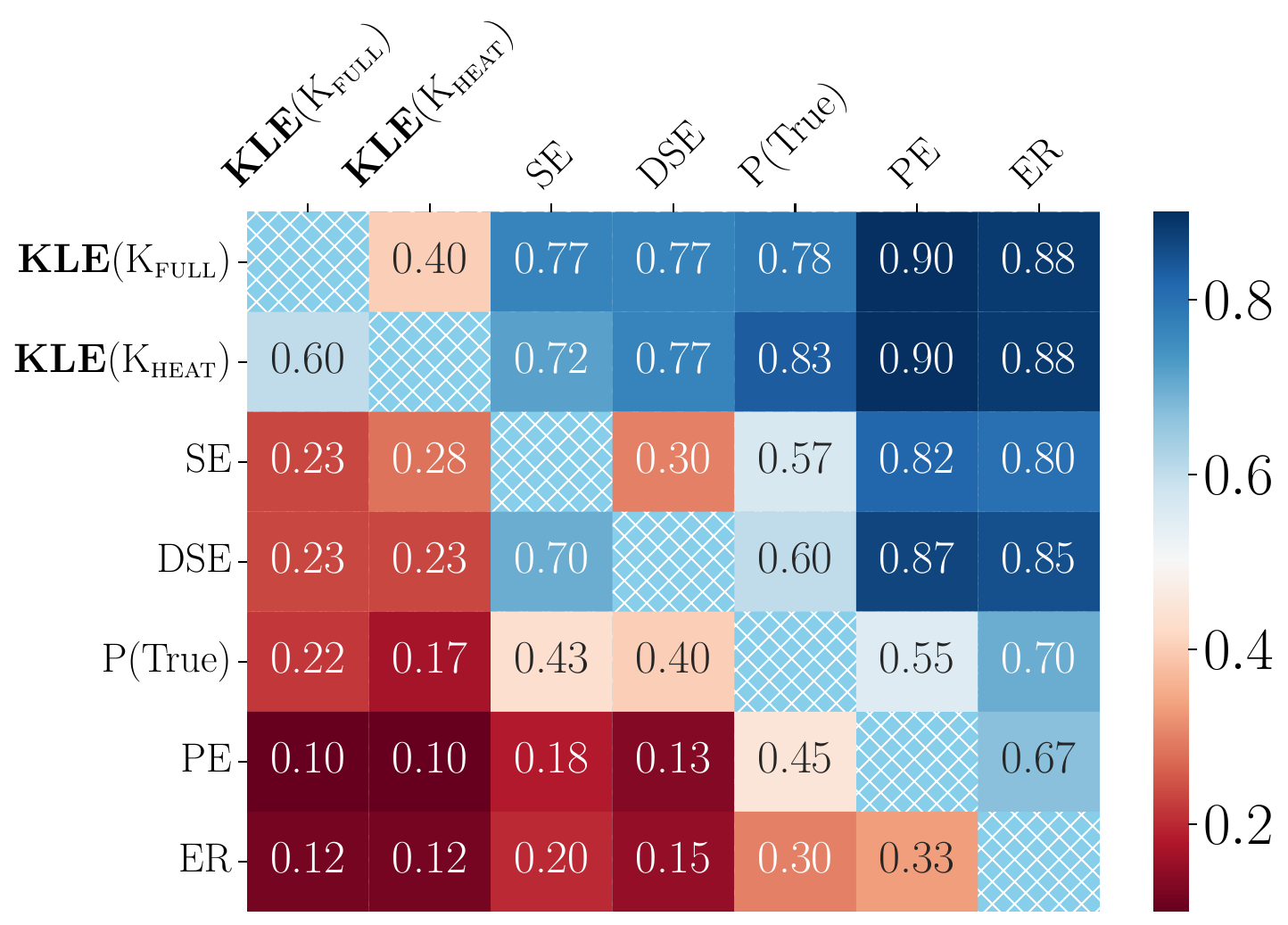}
  \caption{Win rate measured with AUROC}
  \label{fig:auroc_heatmap_best_all}
\end{subfigure}%
\begin{subfigure}{.5\textwidth}
  \centering
  \includegraphics[width=\linewidth]{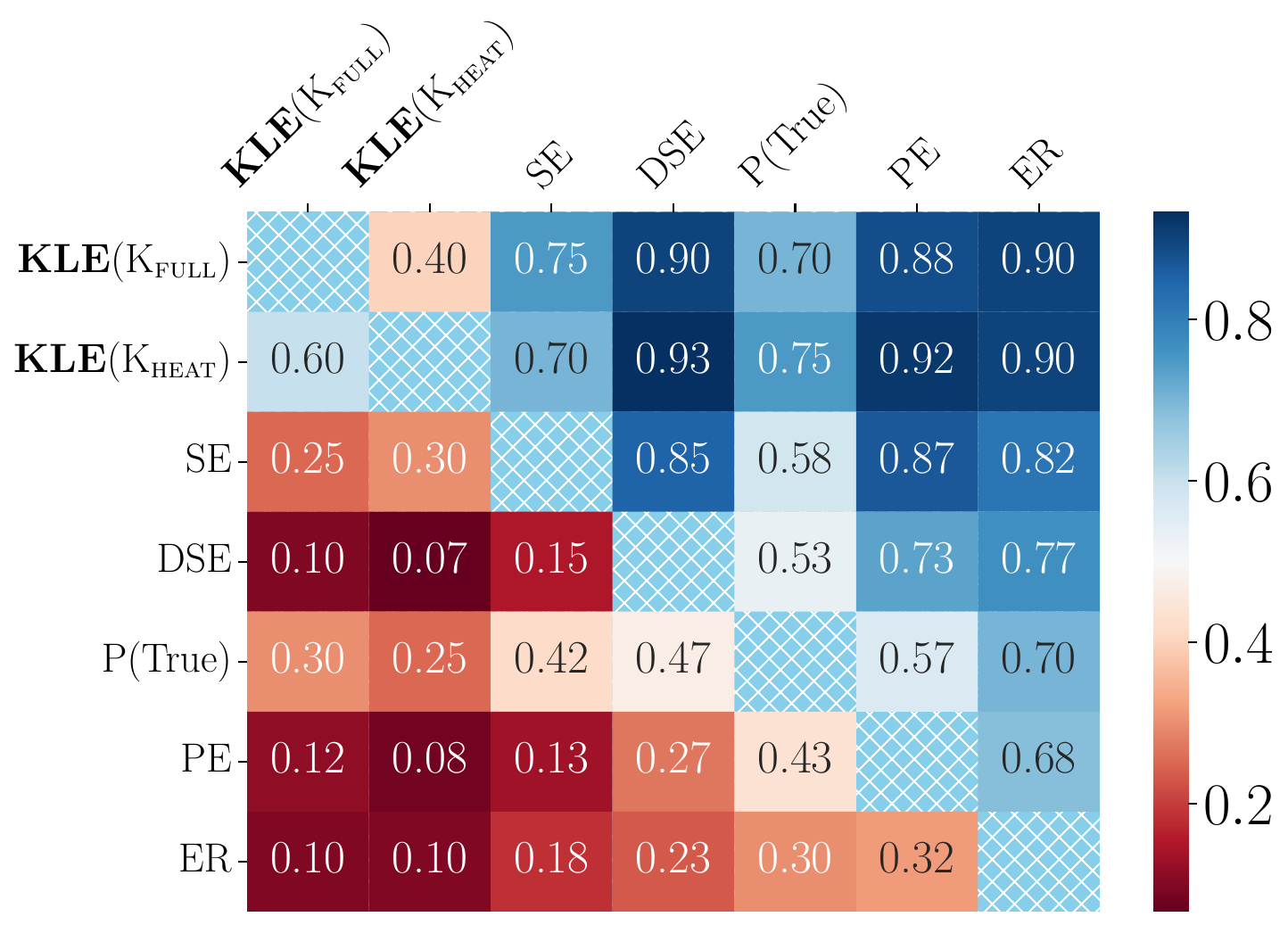}
  \caption{Win rate measured with AUARC}
  \label{fig:auarc_heatmap_best_all}
\end{subfigure}
\caption{Summary of \textbf{60} experimental scenarios. Each cell contains the fraction of experiments where a method from a row outperforms a method from a column. Our methods are labeled $\KLE(\cdot)$. Values larger than or equal to $0.62$ correspond to the significance level $p<0.05$ according to the binomial statistical significance test.}
\label{fig:heatmaps_best_all}
\end{figure}

\textbf{\HO{KLE outperforms previous methods.}}
We compare the performance of UQ methods over 60 scenarios (12 models, five datasets). 
\Cref{fig:heatmaps_best_all} shows the heatmaps of pairwise win rates.
% We observe that
% kernel-based approaches outperform the existing methods in overall comparison. 
% For this, we analyze the pairwise heatmaps of win-rates in . 
We observe that both our methods, $\KLE(K_{\textsc{heat}})$ and $\KLE(K_{\textsc{full}})$, are superior to the baselines. Furthermore, \cref{table:exp_results_llama_falcon} shows the detailed results for the two largest models from our experiments, Llama 2 70B Chat and Falcon 40B Instruct. The results show that for largest models our method consistently achieves best results compared to baselines. In \cref{fig:app:all_results_non_instr} and \cref{fig:app:all_results_instr}, we show the experimental results for all considered models. Importantly, our best method, $\operatorname{KLE}(K_{\textsc{heat}})$, does not require token-level probabilities from a model, and works in black-box scenarios.

\textbf{\HO{KLE hyperparameters can be selected without validation sets.}}
We compare the strategies of hyperparameter selection from  \cref{section:semantic_graph_kernels}: entropy convergence plots and validation sets (100 samples per dataset except for SVAMP, where we used default hyperparameters). We observe that default hyperparameters achieve similar results as selecting hyperparameters from validation sets and conclude that choosing default hyperparameters from entropy convergence plots is a good way to select hyperparameters in practice. In \cref{fig:kernel_comparison}, we compare the two strategies for selecting hyperparameters, and see that the ranking of the methods remains stable and the pairwise win-rates are similar for both methods.

\begin{wrapfigure}{r}{0.45\textwidth}
    \vspace{-.7cm}
     \centering
     \includegraphics[width=.45\textwidth]{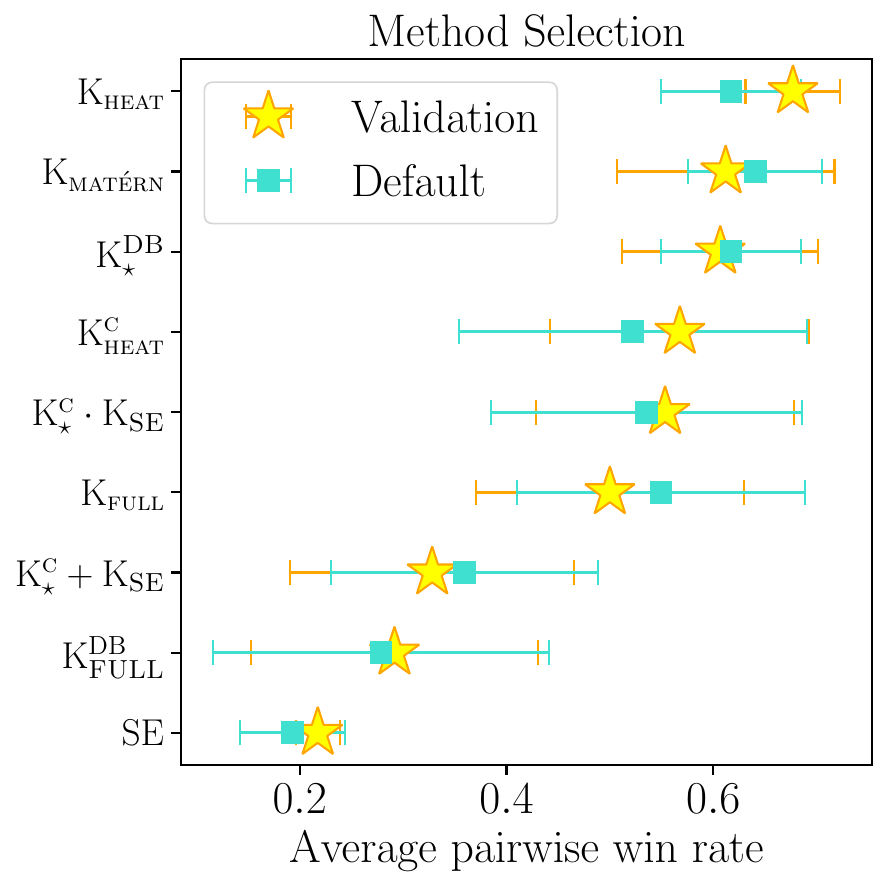}
     \caption{Comparison of various design choices for semantic graph kernels. \textcolor{yellow}{\faStar} represents the best hyperparameters and \textcolor{turquoise}{\faSquare} -- defaults. Error bars are twice the standard error. Summary of 48 experiments.}
     \label{fig:kernel_comparison}
     \vspace{-.5cm}
\end{wrapfigure}

\textbf{\HO{Many design choices outperform existing methods, the best is $\KLE(K_{\textsc{heat}})$.}}
Next, in \cref{fig:kernel_comparison}, we compare several design choices for KLE: choosing a kernel (heat or Mat\'ern), using KLE-c, combining kernels via a weighted sum or product, and using the probabilities returned by DeBERTa for edge weights. The superscript indicates the type of a graph: no superscript indicates a weighted graph as described above, \textsc{DB} means weights are assigned using probabilities from DeBERTa, and \textsc{C} means a weighted graph over clusters ($\operatorname{KLE-c}$). The subscript indicates the semantic kernels: \textsc{SE} stands for a diagonal kernel with semantic probabilities, \textsc{heat} and \textsc{mat\'ern} for the type of kernel, and $\star$ for the best of Heat and Mat\'ern kernels. We observe that even though all design choices outperform SE, the heat kernel over a weighted semantic graph, $\KLE(K_{\textsc{heat}})$, was overall the best. Additionally, we notice that the methods based on token likelihoods are performing better for non-instruction tuned models, and we can practically recommend including semantic probabilities (e.g., use variations of $K_{\textsc{full}}$) if KLE is used in non-instruction tuned scenarios (see \cref{app:fig:instr_vs_non_instr}).

\textbf{\HO{KLE is better in practice because it captures more fine-grained semantic relations than SE.}}
The performance of KLE improves over SE because in complex free-form language generation scenarios, such as those studied here, LLMs can generate similar but not strictly equal answers. SE assigns these to separate clusters, predicting high entropy. By contrast, our method can account for semantic similarities using the kernel metric in the space of meanings over generated texts, and predict reduced uncertainty if necessary. 
We give a detailed illustrative example for which KLE provides better uncertainty estimates than SE from the NQ Open dataset in \cref{fig:app:example}. 
% We consider a detailed and common of a case when KLE provides better uncertainty estimates than SE from the NQ dataset in \cref{fig:app:example}. 

\section{Discussion}
\label{section:discussion}
Measuring semantic uncertainty in LLMs is a challenging and important problem. It requires navigating the semantic space of the answers, and we have suggested a method, KLE, that encodes a similarity measure in this space via semantic kernels. KLE allows for fine-grained estimation of uncertainty and is an expressive generalization of semantic entropy. We provided several specific design choices by defining NLI-based semantic graphs and kernels, and studying kernel hyperparameters. We have evaluated KLE across various domains of natural language generation, and it has demonstrated superior performance compared to the previous methods. Our method works both for white- and black-box settings, enabling its application to a wide variety of practical scenarios. 
We hope to inspire more work that moves from semantic \emph{equivalence} to semantic \emph{similarity} for estimating semantic uncertainty in LLMs.

\textbf{Broader Impact.} Our work advances the progress toward safer and more reliable uses of LLMs. KLE can positively impact areas that involve using LLMs by providing more accurate uncertainty estimates, which can filter out a proportion of erroneous outputs. 

\textbf{Limitations.} One limitation of the proposed method is that it requires multiple samples from an LLM, which generally increases the generation cost. However, in safety-critical tasks, the potential cost of hallucination should outweigh the cost of sampling multiple answers, so reliable uncertainty quantification via KLE should always be worthwhile. Additionally, we study semantic kernels derived from NLI-based semantic graphs, but other semantic kernels warrant investigation, such as kernels on embeddings. Lastly, the NLG landscape is highly diverse, and the method should be carefully evaluated for other potential applications of LLM, such as code generation.

\phantomsection%
\begin{ack}
% TODO
This work was supported by the Research Council of Finland (Flagship programme: Finnish Center for Artificial Intelligence FCAI, and grants 352986, 358246) and EU (H2020 grant 101016775 and NextGenerationEU).
\end{ack}
\addcontentsline{toc}{section}{References}
\begingroup
\small
\bibliographystyle{abbrvnat}
\bibliography{main}

\begin{thebibliography}{76}
\providecommand{\natexlab}[1]{#1}
\providecommand{\url}[1]{\texttt{#1}}
\expandafter\ifx\csname urlstyle\endcsname\relax
  \providecommand{\doi}[1]{doi: #1}\else
  \providecommand{\doi}{doi: \begingroup \urlstyle{rm}\Url}\fi

\bibitem[Aaronson(2022)]{aaronson2022introduction}
S.~Aaronson.
\newblock Introduction to quantum information science {II} lecture notes, 2022.

\bibitem[Aichberger et~al.(2024)Aichberger, Schweighofer, Ielanskyi, and Hochreiter]{aichberger2024many}
L.~Aichberger, K.~Schweighofer, M.~Ielanskyi, and S.~Hochreiter.
\newblock How many opinions does your llm have? improving uncertainty estimation in nlg.
\newblock In \emph{ICLR 2024 Workshop on Secure and Trustworthy Large Language Models}, 2024.

\bibitem[Almazrouei et~al.(2023)Almazrouei, Alobeidli, Alshamsi, Cappelli, Cojocaru, Debbah, Goffinet, Hesslow, Launay, Malartic, et~al.]{almazrouei2023falcon}
E.~Almazrouei, H.~Alobeidli, A.~Alshamsi, A.~Cappelli, R.~Cojocaru, M.~Debbah, {\'E}.~Goffinet, D.~Hesslow, J.~Launay, Q.~Malartic, et~al.
\newblock The falcon series of open language models.
\newblock \emph{arXiv preprint arXiv:2311.16867}, 2023.

\bibitem[Bach(2022)]{bach2022information}
F.~Bach.
\newblock Information theory with kernel methods.
\newblock \emph{IEEE Transactions on Information Theory}, 69\penalty0 (2):\penalty0 752--775, 2022.

\bibitem[Bengtsson and {\.Z}yczkowski(2017)]{bengtsson2017geometry}
I.~Bengtsson and K.~{\.Z}yczkowski.
\newblock \emph{Geometry of quantum states: an introduction to quantum entanglement}.
\newblock Cambridge university press, 2017.

\bibitem[Borovitskiy et~al.(2021)Borovitskiy, Azangulov, Terenin, Mostowsky, Deisenroth, and Durrande]{borovitskiy2021matern}
V.~Borovitskiy, I.~Azangulov, A.~Terenin, P.~Mostowsky, M.~Deisenroth, and N.~Durrande.
\newblock Mat{\'e}rn {G}aussian processes on graphs.
\newblock In \emph{International Conference on Artificial Intelligence and Statistics}, pages 2593--2601. PMLR, 2021.

\bibitem[Budanitsky and Hirst(2006)]{budanitsky2006evaluating}
A.~Budanitsky and G.~Hirst.
\newblock Evaluating wordnet-based measures of lexical semantic relatedness.
\newblock \emph{Computational linguistics}, 32\penalty0 (1):\penalty0 13--47, 2006.

\bibitem[Burns et~al.(2022)Burns, Ye, Klein, and Steinhardt]{burns2024discovering}
C.~Burns, H.~Ye, D.~Klein, and J.~Steinhardt.
\newblock Discovering latent knowledge in language models without supervision, 2022.

\bibitem[Chen and Mueller(2023)]{chen2023quantifying}
J.~Chen and J.~Mueller.
\newblock Quantifying uncertainty in answers from any language model via intrinsic and extrinsic confidence assessment.
\newblock \emph{arXiv preprint arXiv:2308.16175}, 2023.

\bibitem[Chung(1997)]{chung1997spectral}
F.~R. Chung.
\newblock \emph{Spectral graph theory}, volume~92.
\newblock American Mathematical Soc., 1997.

\bibitem[Clusmann et~al.(2023)Clusmann, Kolbinger, Muti, Carrero, Eckardt, Laleh, L{\"o}ffler, Schwarzkopf, Unger, Veldhuizen, et~al.]{clusmann2023future}
J.~Clusmann, F.~R. Kolbinger, H.~S. Muti, Z.~I. Carrero, J.-N. Eckardt, N.~G. Laleh, C.~M.~L. L{\"o}ffler, S.-C. Schwarzkopf, M.~Unger, G.~P. Veldhuizen, et~al.
\newblock The future landscape of large language models in medicine.
\newblock \emph{Communications medicine}, 3\penalty0 (1):\penalty0 141, 2023.

\bibitem[Cohen et~al.(2023)Cohen, Hamri, Geva, and Globerson]{cohen2023lm}
R.~Cohen, M.~Hamri, M.~Geva, and A.~Globerson.
\newblock {LM} vs {LM}: Detecting factual errors via cross examination.
\newblock \emph{arXiv preprint arXiv:2305.13281}, 2023.

\bibitem[Cole et~al.(2023)Cole, Zhang, Gillick, Eisenschlos, Dhingra, and Eisenstein]{cole2023selectively}
J.~R. Cole, M.~J. Zhang, D.~Gillick, J.~M. Eisenschlos, B.~Dhingra, and J.~Eisenstein.
\newblock Selectively answering ambiguous questions.
\newblock \emph{Conference on Empirical Methods in Natural Language Processing}, 2023.

\bibitem[Crystal(2018)]{crystal2018cambridge}
D.~Crystal.
\newblock \emph{The Cambridge encyclopedia of the English language}.
\newblock Cambridge university press, 2018.

\bibitem[Desai and Durrett(2020)]{desai2020calibration}
S.~Desai and G.~Durrett.
\newblock Calibration of pre-trained transformers.
\newblock \emph{arXiv preprint arXiv:2003.07892}, 2020.

\bibitem[Farquhar et~al.(2024)Farquhar, Kossen, Kuhn, and Gal]{nature_paper}
S.~Farquhar, J.~Kossen, L.~Kuhn, and Y.~Gal.
\newblock Personal Communication, 2024.

\bibitem[Feldman et~al.(2023)Feldman, Foulds, and Pan]{feldman2023trapping}
P.~Feldman, J.~R. Foulds, and S.~Pan.
\newblock Trapping {LLM} hallucinations using tagged context prompts.
\newblock \emph{arXiv preprint arXiv:2306.06085}, 2023.

\bibitem[Filippova(2020)]{filippova2020controlled}
K.~Filippova.
\newblock Controlled hallucinations: Learning to generate faithfully from noisy data.
\newblock \emph{arXiv preprint arXiv:2010.05873}, 2020.

\bibitem[Friedman and Dieng(2023)]{friedman2023vendi}
D.~Friedman and A.~B. Dieng.
\newblock The {V}endi score: A diversity evaluation metric for machine learning.
\newblock \emph{Transactions on Machine Learning Research}, 2023.

\bibitem[Gal and Ghahramani(2016)]{gal2016dropout}
Y.~Gal and Z.~Ghahramani.
\newblock Dropout as a {B}ayesian approximation: Representing model uncertainty in deep learning.
\newblock In \emph{International Conference on Machine Learning}, pages 1050--1059. PMLR, 2016.

\bibitem[Ganguli et~al.(2023)Ganguli, Askell, Schiefer, Liao, Luko{\v{s}}i{\=u}t{\.e}, Chen, Goldie, Mirhoseini, Olsson, Hernandez, et~al.]{ganguli2023capacity}
D.~Ganguli, A.~Askell, N.~Schiefer, T.~I. Liao, K.~Luko{\v{s}}i{\=u}t{\.e}, A.~Chen, A.~Goldie, A.~Mirhoseini, C.~Olsson, D.~Hernandez, et~al.
\newblock The capacity for moral self-correction in large language models.
\newblock \emph{arXiv preprint arXiv:2302.07459}, 2023.

\bibitem[He et~al.(2020)He, Liu, Gao, and Chen]{he2020deberta}
P.~He, X.~Liu, J.~Gao, and W.~Chen.
\newblock Deberta: Decoding-enhanced bert with disentangled attention.
\newblock \emph{arXiv preprint arXiv:2006.03654}, 2020.

\bibitem[Hendrycks et~al.(2021)Hendrycks, Carlini, Schulman, and Steinhardt]{hendrycks2021unsolved}
D.~Hendrycks, N.~Carlini, J.~Schulman, and J.~Steinhardt.
\newblock Unsolved problems in {ML} safety.
\newblock \emph{arXiv preprint arXiv:2109.13916}, 2021.

\bibitem[Hernandez et~al.(2023)Hernandez, Li, and Andreas]{hernandez2023measuring}
E.~Hernandez, B.~Z. Li, and J.~Andreas.
\newblock Measuring and manipulating knowledge representations in language models.
\newblock \emph{arXiv preprint arXiv:2304.00740}, 2023.

\bibitem[Horn and Johnson(2012)]{horn2012matrix}
R.~A. Horn and C.~R. Johnson.
\newblock \emph{Matrix analysis}.
\newblock Cambridge university press, 2012.

\bibitem[Ji et~al.(2023)Ji, Lee, Frieske, Yu, Su, Xu, Ishii, Bang, Madotto, and Fung]{ji2023survey}
Z.~Ji, N.~Lee, R.~Frieske, T.~Yu, D.~Su, Y.~Xu, E.~Ishii, Y.~J. Bang, A.~Madotto, and P.~Fung.
\newblock Survey of hallucination in natural language generation.
\newblock \emph{ACM Computing Surveys}, 55\penalty0 (12):\penalty0 1--38, 2023.

\bibitem[Jiang et~al.(2023)Jiang, Sablayrolles, Mensch, Bamford, Chaplot, Casas, Bressand, Lengyel, Lample, Saulnier, et~al.]{jiang2023mistral}
A.~Q. Jiang, A.~Sablayrolles, A.~Mensch, C.~Bamford, D.~S. Chaplot, D.~d.~l. Casas, F.~Bressand, G.~Lengyel, G.~Lample, L.~Saulnier, et~al.
\newblock Mistral 7b.
\newblock \emph{arXiv preprint arXiv:2310.06825}, 2023.

\bibitem[Jiang et~al.(2021)Jiang, Araki, Ding, and Neubig]{jiang2021can}
Z.~Jiang, J.~Araki, H.~Ding, and G.~Neubig.
\newblock How can we know when language models know? on the calibration of language models for question answering.
\newblock \emph{Transactions of the Association for Computational Linguistics}, 9:\penalty0 962--977, 2021.

\bibitem[Joshi et~al.(2017)Joshi, Choi, Weld, and Zettlemoyer]{joshi2017triviaqa}
M.~Joshi, E.~Choi, D.~S. Weld, and L.~Zettlemoyer.
\newblock Trivia{QA}: A large scale distantly supervised challenge dataset for reading comprehension.
\newblock \emph{arXiv preprint arXiv:1705.03551}, 2017.

\bibitem[Kadavath et~al.(2022)Kadavath, Conerly, Askell, Henighan, Drain, Perez, Schiefer, Hatfield-Dodds, DasSarma, Tran-Johnson, et~al.]{kadavath2022language}
S.~Kadavath, T.~Conerly, A.~Askell, T.~Henighan, D.~Drain, E.~Perez, N.~Schiefer, Z.~Hatfield-Dodds, N.~DasSarma, E.~Tran-Johnson, et~al.
\newblock Language models (mostly) know what they know.
\newblock \emph{arXiv preprint arXiv:2207.05221}, 2022.

\bibitem[Kang et~al.(2024)Kang, Wallace, Tomlin, Kumar, and Levine]{kang2024unfamiliar}
K.~Kang, E.~Wallace, C.~Tomlin, A.~Kumar, and S.~Levine.
\newblock Unfamiliar finetuning examples control how language models hallucinate.
\newblock \emph{arXiv preprint arXiv:2403.05612}, 2024.

\bibitem[Kasneci et~al.(2023)Kasneci, Se{\ss}ler, K{\"u}chemann, Bannert, Dementieva, Fischer, Gasser, Groh, G{\"u}nnemann, H{\"u}llermeier, et~al.]{kasneci2023chatgpt}
E.~Kasneci, K.~Se{\ss}ler, S.~K{\"u}chemann, M.~Bannert, D.~Dementieva, F.~Fischer, U.~Gasser, G.~Groh, S.~G{\"u}nnemann, E.~H{\"u}llermeier, et~al.
\newblock Chatgpt for good? on opportunities and challenges of large language models for education.
\newblock \emph{Learning and individual differences}, 103:\penalty0 102274, 2023.

\bibitem[Kim et~al.(2023)Kim, Kang, Hwang, Shin, and Rhee]{kim2023vne}
J.~Kim, S.~Kang, D.~Hwang, J.~Shin, and W.~Rhee.
\newblock Vne: An effective method for improving deep representation by manipulating eigenvalue distribution.
\newblock In \emph{Proceedings of the IEEE/CVF Conference on Computer Vision and Pattern Recognition}, pages 3799--3810, 2023.

\bibitem[Kondor and Lafferty(2002)]{kondor2002diffusion}
R.~I. Kondor and J.~Lafferty.
\newblock Diffusion kernels on graphs and other discrete structures.
\newblock In \emph{Proceedings of the 19th international conference on machine learning}, volume 2002, pages 315--322, 2002.

\bibitem[Krithara et~al.(2023)Krithara, Nentidis, Bougiatiotis, and Paliouras]{krithara2023bioasq}
A.~Krithara, A.~Nentidis, K.~Bougiatiotis, and G.~Paliouras.
\newblock {BioASQ-QA}: A manually curated corpus for biomedical question answering.
\newblock \emph{Scientific Data}, 10\penalty0 (1):\penalty0 170, 2023.

\bibitem[Kuhn et~al.(2023)Kuhn, Gal, and Farquhar]{kuhn2023semantic}
L.~Kuhn, Y.~Gal, and S.~Farquhar.
\newblock Semantic uncertainty: Linguistic invariances for uncertainty estimation in natural language generation.
\newblock \emph{arXiv preprint arXiv:2302.09664}, 2023.

\bibitem[Kumar and Sarawagi(2019)]{kumar2019calibration}
A.~Kumar and S.~Sarawagi.
\newblock Calibration of encoder decoder models for neural machine translation.
\newblock \emph{arXiv preprint arXiv:1903.00802}, 2019.

\bibitem[Kwiatkowski et~al.(2019)Kwiatkowski, Palomaki, Redfield, Collins, Parikh, Alberti, Epstein, Polosukhin, Devlin, Lee, et~al.]{kwiatkowski2019natural}
T.~Kwiatkowski, J.~Palomaki, O.~Redfield, M.~Collins, A.~Parikh, C.~Alberti, D.~Epstein, I.~Polosukhin, J.~Devlin, K.~Lee, et~al.
\newblock Natural questions: a benchmark for question answering research.
\newblock \emph{Transactions of the Association for Computational Linguistics}, 7:\penalty0 453--466, 2019.

\bibitem[Lakshminarayanan et~al.(2017)Lakshminarayanan, Pritzel, and Blundell]{lakshminarayanan2017simple}
B.~Lakshminarayanan, A.~Pritzel, and C.~Blundell.
\newblock Simple and scalable predictive uncertainty estimation using deep ensembles.
\newblock \emph{Advances in neural information processing systems}, 30, 2017.

\bibitem[Le et~al.(2022)Le, Wang, Gotmare, Savarese, and Hoi]{le2022coderl}
H.~Le, Y.~Wang, A.~D. Gotmare, S.~Savarese, and S.~C.~H. Hoi.
\newblock Coderl: Mastering code generation through pretrained models and deep reinforcement learning.
\newblock \emph{Advances in Neural Information Processing Systems}, 35:\penalty0 21314--21328, 2022.

\bibitem[Li et~al.(2024)Li, Patel, Vi{\'e}gas, Pfister, and Wattenberg]{li2024inference}
K.~Li, O.~Patel, F.~Vi{\'e}gas, H.~Pfister, and M.~Wattenberg.
\newblock Inference-time intervention: Eliciting truthful answers from a language model.
\newblock \emph{Advances in Neural Information Processing Systems}, 36, 2024.

\bibitem[Li et~al.(2023)Li, Zhao, Chia, Ding, Joty, Poria, and Bing]{li2023chain}
X.~Li, R.~Zhao, Y.~K. Chia, B.~Ding, S.~Joty, S.~Poria, and L.~Bing.
\newblock Chain-of-knowledge: Grounding large language models via dynamic knowledge adapting over heterogeneous sources.
\newblock In \emph{The Twelfth International Conference on Learning Representations}, 2023.

\bibitem[Lin et~al.(2022)Lin, Hilton, and Evans]{lin2022teaching}
S.~Lin, J.~Hilton, and O.~Evans.
\newblock Teaching models to express their uncertainty in words.
\newblock \emph{arXiv preprint arXiv:2205.14334}, 2022.

\bibitem[Lin et~al.(2023)Lin, Trivedi, and Sun]{lin2023generating}
Z.~Lin, S.~Trivedi, and J.~Sun.
\newblock Generating with confidence: Uncertainty quantification for black-box large language models.
\newblock \emph{arXiv preprint arXiv:2305.19187}, 2023.

\bibitem[Liu et~al.(2020)Liu, Lin, Padhy, Tran, Bedrax~Weiss, and Lakshminarayanan]{liu2020simple}
J.~Liu, Z.~Lin, S.~Padhy, D.~Tran, T.~Bedrax~Weiss, and B.~Lakshminarayanan.
\newblock Simple and principled uncertainty estimation with deterministic deep learning via distance awareness.
\newblock \emph{Advances in neural information processing systems}, 33:\penalty0 7498--7512, 2020.

\bibitem[Liu et~al.(2023)Liu, Xing, and Zou]{liu2023context}
S.~Liu, L.~Xing, and J.~Zou.
\newblock In-context vectors: Making in context learning more effective and controllable through latent space steering.
\newblock \emph{arXiv preprint arXiv:2311.06668}, 2023.

\bibitem[MacDiarmid et~al.(2024)MacDiarmid, Maxwell, Schiefer, Mu, Kaplan, Duvenaud, Bowman, Tamkin, Perez, Sharma, Denison, and Hubinger]{macdiarmid2024sleeperagentprobes}
M.~MacDiarmid, T.~Maxwell, N.~Schiefer, J.~Mu, J.~Kaplan, D.~Duvenaud, S.~Bowman, A.~Tamkin, E.~Perez, M.~Sharma, C.~Denison, and E.~Hubinger.
\newblock Simple probes can catch sleeper agents, 2024.
\newblock URL \url{https://www.anthropic.com/news/probes-catch-sleeper-agents}.

\bibitem[MacKay(2003)]{mackay2003information}
D.~J. MacKay.
\newblock \emph{Information theory, inference and learning algorithms}.
\newblock Cambridge university press, 2003.

\bibitem[Malinin and Gales(2020)]{malinin2020uncertainty}
A.~Malinin and M.~Gales.
\newblock Uncertainty estimation in autoregressive structured prediction.
\newblock \emph{arXiv preprint arXiv:2002.07650}, 2020.

\bibitem[Manakul et~al.(2023)Manakul, Liusie, and Gales]{manakul2023selfcheckgpt}
P.~Manakul, A.~Liusie, and M.~J. Gales.
\newblock {SelfCheckGPT}: Zero-resource black-box hallucination detection for generative large language models.
\newblock In \emph{Conference on Empirical Methods in Natural Language Processing}, 2023.

\bibitem[Maynez et~al.(2020)Maynez, Narayan, Bohnet, and McDonald]{maynez2020faithfulness}
J.~Maynez, S.~Narayan, B.~Bohnet, and R.~McDonald.
\newblock On faithfulness and factuality in abstractive summarization.
\newblock \emph{arXiv preprint arXiv:2005.00661}, 2020.

\bibitem[Mielke et~al.(2020)Mielke, Szlam, Boureau, and Dinan]{mielke2020linguistic}
S.~J. Mielke, A.~Szlam, Y.-L. Boureau, and E.~Dinan.
\newblock Linguistic calibration through metacognition: aligning dialogue agent responses with expected correctness.
\newblock \emph{arXiv preprint arXiv:2012.14983}, 11, 2020.

\bibitem[Mielke et~al.(2022)Mielke, Szlam, Dinan, and Boureau]{mielke2022reducing}
S.~J. Mielke, A.~Szlam, E.~Dinan, and Y.-L. Boureau.
\newblock Reducing conversational agents’ overconfidence through linguistic calibration.
\newblock \emph{Transactions of the Association for Computational Linguistics}, 10:\penalty0 857--872, 2022.

\bibitem[Mukhoti et~al.(2023)Mukhoti, Kirsch, van Amersfoort, Torr, and Gal]{mukhoti2023deep}
J.~Mukhoti, A.~Kirsch, J.~van Amersfoort, P.~H. Torr, and Y.~Gal.
\newblock Deep deterministic uncertainty: A new simple baseline.
\newblock In \emph{Proceedings of the IEEE/CVF Conference on Computer Vision and Pattern Recognition}, pages 24384--24394, 2023.

\bibitem[Nadeem et~al.(2009)Nadeem, Zucker, and Hanczar]{nadeem2009accuracy}
M.~S.~A. Nadeem, J.-D. Zucker, and B.~Hanczar.
\newblock Accuracy-rejection curves ({ARC}s) for comparing classification methods with a reject option.
\newblock In \emph{Machine Learning in Systems Biology}, pages 65--81. PMLR, 2009.

\bibitem[Nikitin et~al.(2022)Nikitin, John, Solin, and Kaski]{nikitin2022non}
A.~V. Nikitin, S.~John, A.~Solin, and S.~Kaski.
\newblock Non-separable spatio-temporal graph kernels via {SPDE}s.
\newblock In \emph{International Conference on Artificial Intelligence and Statistics}, pages 10640--10660. PMLR, 2022.

\bibitem[OpenAI(2023)]{achiam2023gpt}
OpenAI.
\newblock {GPT}-4 technical report.
\newblock 2023.

\bibitem[Ovadia et~al.(2019)Ovadia, Fertig, Ren, Nado, Sculley, Nowozin, Dillon, Lakshminarayanan, and Snoek]{ovadia2019can}
Y.~Ovadia, E.~Fertig, J.~Ren, Z.~Nado, D.~Sculley, S.~Nowozin, J.~Dillon, B.~Lakshminarayanan, and J.~Snoek.
\newblock Can you trust your model's uncertainty? evaluating predictive uncertainty under dataset shift.
\newblock \emph{Advances in neural information processing systems}, 32, 2019.

\bibitem[Patel et~al.(2021)Patel, Bhattamishra, and Goyal]{patel-etal-2021-nlp}
A.~Patel, S.~Bhattamishra, and N.~Goyal.
\newblock Are {NLP} models really able to solve simple math word problems?
\newblock In \emph{Proceedings of the 2021 Conference of the North American Chapter of the Association for Computational Linguistics: Human Language Technologies}, pages 2080--2094, Online, June 2021. Association for Computational Linguistics.
\newblock \doi{10.18653/v1/2021.naacl-main.168}.
\newblock URL \url{https://aclanthology.org/2021.naacl-main.168}.

\bibitem[Quach et~al.(2023)Quach, Fisch, Schuster, Yala, Sohn, Jaakkola, and Barzilay]{quach2023conformal}
V.~Quach, A.~Fisch, T.~Schuster, A.~Yala, J.~H. Sohn, T.~S. Jaakkola, and R.~Barzilay.
\newblock Conformal language modeling.
\newblock \emph{arXiv preprint arXiv:2306.10193}, 2023.

\bibitem[Rajpurkar et~al.(2018)Rajpurkar, Jia, and Liang]{rajpurkar2018know}
P.~Rajpurkar, R.~Jia, and P.~Liang.
\newblock Know what you don't know: Unanswerable questions for {SQuAD}.
\newblock \emph{arXiv preprint arXiv:1806.03822}, 2018.

\bibitem[Ren et~al.(2023)Ren, Zhao, Vu, Liu, and Lakshminarayanan]{ren2023self}
J.~Ren, Y.~Zhao, T.~Vu, P.~J. Liu, and B.~Lakshminarayanan.
\newblock Self-evaluation improves selective generation in large language models.
\newblock \emph{arXiv preprint arXiv:2312.09300}, 2023.

\bibitem[Roy and Vetterli(2007)]{roy2007effective}
O.~Roy and M.~Vetterli.
\newblock The effective rank: A measure of effective dimensionality.
\newblock In \emph{2007 15th European signal processing conference}, pages 606--610. IEEE, 2007.

\bibitem[Shannon(1951)]{shannon1951prediction}
C.~E. Shannon.
\newblock Prediction and entropy of printed english.
\newblock \emph{Bell system technical journal}, 30\penalty0 (1):\penalty0 50--64, 1951.

\bibitem[Team et~al.(2023)Team, Anil, Borgeaud, Wu, Alayrac, Yu, Soricut, Schalkwyk, Dai, Hauth, et~al.]{team2023gemini}
G.~Team, R.~Anil, S.~Borgeaud, Y.~Wu, J.-B. Alayrac, J.~Yu, R.~Soricut, J.~Schalkwyk, A.~M. Dai, A.~Hauth, et~al.
\newblock Gemini: a family of highly capable multimodal models.
\newblock 2023.

\bibitem[Tian et~al.(2023{\natexlab{a}})Tian, Mitchell, Yao, Manning, and Finn]{tian2023finetuning}
K.~Tian, E.~Mitchell, H.~Yao, C.~D. Manning, and C.~Finn.
\newblock Fine-tuning language models for factuality.
\newblock \emph{arXiv}, 2023{\natexlab{a}}.

\bibitem[Tian et~al.(2023{\natexlab{b}})Tian, Mitchell, Zhou, Sharma, Rafailov, Yao, Finn, and Manning]{tian2023just}
K.~Tian, E.~Mitchell, A.~Zhou, A.~Sharma, R.~Rafailov, H.~Yao, C.~Finn, and C.~D. Manning.
\newblock Just ask for calibration: Strategies for eliciting calibrated confidence scores from language models fine-tuned with human feedback.
\newblock \emph{arXiv preprint arXiv:2305.14975}, 2023{\natexlab{b}}.

\bibitem[Touvron et~al.(2023)Touvron, Martin, Stone, Albert, Almahairi, Babaei, Bashlykov, Batra, Bhargava, Bhosale, et~al.]{touvron2023llama}
H.~Touvron, L.~Martin, K.~Stone, P.~Albert, A.~Almahairi, Y.~Babaei, N.~Bashlykov, S.~Batra, P.~Bhargava, S.~Bhosale, et~al.
\newblock Llama 2: Open foundation and fine-tuned chat models.
\newblock \emph{arXiv preprint arXiv:2307.09288}, 2023.

\bibitem[Varshney et~al.(2023)Varshney, Yao, Zhang, Chen, and Yu]{varshney2023stitch}
N.~Varshney, W.~Yao, H.~Zhang, J.~Chen, and D.~Yu.
\newblock A stitch in time saves nine: Detecting and mitigating hallucinations of llms by validating low-confidence generation.
\newblock \emph{arXiv preprint arXiv:2307.03987}, 2023.

\bibitem[Von~Luxburg(2007)]{von2007tutorial}
U.~Von~Luxburg.
\newblock A tutorial on spectral clustering.
\newblock \emph{Statistics and computing}, 17:\penalty0 395--416, 2007.

\bibitem[Von~Neumann(2018)]{von2018mathematical}
J.~Von~Neumann.
\newblock \emph{Mathematical foundations of quantum mechanics: New edition}, volume~53.
\newblock Princeton university press, 2018.

\bibitem[Xiao and Wang(2019)]{xiao2019quantifying}
Y.~Xiao and W.~Y. Wang.
\newblock Quantifying uncertainties in natural language processing tasks.
\newblock In \emph{Proceedings of the AAAI conference on artificial intelligence}, volume~33, pages 7322--7329, 2019.

\bibitem[Xiao and Wang(2021)]{xiao2021hallucination}
Y.~Xiao and W.~Y. Wang.
\newblock On hallucination and predictive uncertainty in conditional language generation.
\newblock \emph{arXiv preprint arXiv:2103.15025}, 2021.

\bibitem[Yang et~al.(2023)Yang, Robeyns, Wang, and Aitchison]{yang2023bayesian}
A.~X. Yang, M.~Robeyns, X.~Wang, and L.~Aitchison.
\newblock Bayesian low-rank adaptation for large language models.
\newblock \emph{arXiv preprint arXiv:2308.13111}, 2023.

\bibitem[Zhang et~al.(2024)Zhang, Liu, Basaldella, and Collier]{zhang2024luq}
C.~Zhang, F.~Liu, M.~Basaldella, and N.~Collier.
\newblock Luq: Long-text uncertainty quantification for llms.
\newblock \emph{arXiv preprint arXiv:2403.20279}, 2024.

\bibitem[Zou et~al.(2023)Zou, Phan, Chen, Campbell, Guo, Ren, Pan, Yin, Mazeika, Dombrowski, et~al.]{zou2023representation}
A.~Zou, L.~Phan, S.~Chen, J.~Campbell, P.~Guo, R.~Ren, A.~Pan, X.~Yin, M.~Mazeika, A.-K. Dombrowski, et~al.
\newblock Representation engineering: A top-down approach to ai transparency.
\newblock \emph{arXiv preprint arXiv:2310.01405}, 2023.

\end{thebibliography}
\endgroup

\clearpage
\appendix
\appendix
\counterwithin{figure}{section}
\counterwithin{equation}{section}

\nipstitle{
    {\Large Supplementary Material:} \\ Kernel Language Entropy: Fine-grained Uncertainty Quantification for LLMs from Semantic Similarities}
\pagestyle{empty}

%%%%%%%%%%%%%%%%%%%%%%%%%%%%%%%%%%%%%%%%%%%%%%%%%%%%%%%%%%%%

\appendix

\section{Background}
\subsection{Linear Algebra}

\begin{definition}
    \label{app:definition:kernel}
    For a set $\mathcal{X} \neq \emptyset$, a symmetric function $K: \mathcal{X} \times \mathcal{X} \rightarrow \mathrm{R}$ is called a positive-semidefinite kernel if for all $n>0, x_i \in \mathcal{X}, \alpha_i \in \mathbb{R}$
    \begin{equation}
        \sum_{i=1}^n\sum_{j=1}^n \alpha_i \alpha_j K(x_i, x_j) \geq 0.
    \end{equation}
    For a finite set $\mathcal{X}$, a positive semidefinite kernel is a positive semidefinite matrix of the size $|\mathcal{X}|$.
\end{definition}

\begin{lemma}
For a block diagonal matrix \[  A = 
\begin{pmatrix}
  A_{11} & 0       & 0       & \ldots & 0 \\
  0      & A_{22}  & 0       & \ldots & 0 \\
  0      & 0       & A_{33}  & \ldots & 0 \\
  \vdots & \vdots  & \vdots  & \ddots & \vdots \\
  0      & 0       & 0       & \ldots & A_{nn} \\
\end{pmatrix}
\]
eigenvalues are all eigenvalues of the blocks $A_{ii}$ combined, or equivalently $\det{(A - \lambda I)} = 0 \Leftrightarrow \prod\limits_{i=1}^n\det(A_{ii} - \lambda I) = 0$
\end{lemma}
\begin{proof}
Notice, that a block diagonal matrix can be decomposed into the following product:
\begin{align*}
A =\begin{pmatrix}
  A_{11} & 0            & \ldots & 0 \\
  0      & I_{22}         & \ldots & 0 \\
  \vdots & \vdots  & \ddots & \vdots \\
  0      & 0            & \ldots & I_{nn} \\
\end{pmatrix}
\begin{pmatrix}
  I_{11} & 0            & \ldots & 0 \\
  0      & A_{22}         & \ldots & 0 \\
  \vdots & \vdots  & \ddots & \vdots \\
  0      & 0            & \ldots & I_{nn} \\
\end{pmatrix}
\ldots
&\begin{pmatrix}
  I_{11} & 0            & \ldots & 0 \\
  0      & I_{22}         & \ldots & 0 \\
  \vdots & \vdots  & \ddots & \vdots \\
  0      & 0            & \ldots & A_{nn} \\
\end{pmatrix},
\end{align*}
where $I_{ii}$ are the identity matrices of the same size as $A_{ii}$.

By using the product rule for determinants, we obtain $\det{(A - \lambda I)} = 0 \Leftrightarrow \prod\limits_{i=1}^n\det(A_{ii} - \lambda I) = 0$.
\end{proof}

\begin{lemma}[\citet{horn2012matrix}]
An all-ones matrix $J$ of size $n$ has eigenvalues $\{n, \underbrace{0, \ldots, 0}_{n-1}\}$.
\end{lemma}

\subsection{Discrete Mathematics}
Throughout the text, we often refer to the notion of equivalence relation. We remind readers of the definition of equivalence relation here.
\begin{definition}
    Equivalence relation is a binary relation $E(\cdot, \cdot)$ on a set $\mathcal{X}$, that is for any $x, y, z \in \mathcal{X}$, this relation is
    \begin{enumerate}
        \item reflexive $E(x, x)$,
        \item symmetric $E(x, y) \iff E(y, x)$,
        \item transitive if $E(x, y)$ and $E(y, z)$ then $E(x, z)$.
    \end{enumerate}
    
\end{definition}

\section{Theoretical Results and Proofs}
\label{appendix:additional_theory}

In this section, we prove \cref{proposition:texts_kernel}, for convenience we separate it into two theorems for KLE and KLE-c.

\begin{FrameTheorem}[KLE is a generalization of SE]
    \label{appendix:proposition:texts_kernel}
    For any semantic clustering, there exists a semantic kernel over texts $K_{\text{sem}}(s, s^{\prime})$ such that the VNE of this kernel is equal to semantic entropy (computed as in \cref{eq:mc_se}).
\end{FrameTheorem}

\begin{proof}
    Let us fix an arbitrary semantic clustering over $M$ clusters $\mathcal{C} = \{C_1, \ldots, C_M\}$, with the size of each cluster $m_i$. Now, we will construct a kernel $K$ for an input $x$ such that the von Neumann entropy with this kernel will be equal to the semantic entropy of the texts $\VNE(K) = \SE(x, \mathcal{C})$. Let us consider a block-diagonal kernel $K$. We will denote blocks of $K$ as $K_1, \ldots, K_M$:
    \begin{equation}
        K=
        \begin{pmatrix}
          K_{1} & 0       & 0       & \ldots & 0 \\
          0      & K_{2}  & 0       & \ldots & 0 \\
          0      & 0       & K_{3}  & \ldots & 0 \\
          \vdots & \vdots  & \vdots  & \ddots & \vdots \\
          0      & 0       & 0       & \ldots & K_{M} \\
        \end{pmatrix}
    \end{equation}
    where $M$ corresponds to the number of semantic clusters. The size of each block $K_i$ is $m_i \times m_i$. Note that because $K$ is block-diagonal, it follows that $\VNE(K)=\sum_{i=1}^M \VNE(K_i)$. Consequently, if
    \begin{enumerate}
        \setlength\itemsep{0.01cm}
        \item $\VNE(K_i) = - p(C_i|x) \log p(C_i | x)$,
        \item the sum of eigenvalues of $K_i$ is equal to $p(C_i | x)$,
        \item $K$ is positive semidefinite and unit trace,
    \end{enumerate}
    then $\VNE(K) = \SE(s|x)$. 
    
    Let us define each block as $K_i = \frac{p(C_i | x)}{m_i} J_{m_i}$ where $J_{m_i}$ is an all-ones matrix of size $m_i \times m_i$.
    
    Next, we prove that the desired properties from the list above hold. Indeed, the eigenvalues of $K_i$ are $p(C_i | x)$ with multiplicity one and 0 with multiplicity $m_i - 1$. So, $\VNE(K_i) = - p(C_i|x) \log p(C_i | x)$ (recall that for calculating VN entropy, we assume $0 \log{0} = 0$), and Properties 1 and 2 are fulfilled. $K$ is also symmetric and has non-negative eigenvalues. Thus, Property 3 is fulfilled as well. 
    
    Because $K$ satisfies all properties, we have proven that $\VNE(K(s, x)) = \SE(s|x)$.
\end{proof}

\begin{FrameTheorem}[KLE-c is more general than SE]
        \label{appendix:proposition:sem_clusters_kernel}
    For any semantic clustering, there exists a kernel over semantic clusters $K_s(c, c^{\prime})$ such that the VNE of this kernel is equal to semantic entropy (computed as in \cref{eq:mc_se}).
\end{FrameTheorem}
\begin{proof}
    Analogously to \cref{appendix:proposition:texts_kernel} but with the blocks of size one.
\end{proof}

The theorems not only show that KLE generalizes SE but also provide an explicit form for a semantic kernel that can be used with KLE to recover SE.

\section{Kernel Hyperparameters}
\label{appendix:kernel_hyperparams}
Following the discussion about kernel hyperparameters selection from \cref{section:semantic_graph_kernels}, we visualize entropy convergence plots for Heat kernels in \cref{fig:convergence_plots} and visualize heat and Mat\'ern kernels on 2-d grid in \cref{fig:app:heat_and_matern}. Next, we expand on the question of parameter sensitivity, in \cref{app:fig:comparing_hyperparams}, and whether it is necessary to use a validation set for selecting kernel hyperparameters. We observe that both with reasonable default choices ($t=0.3$, $\alpha=0.5$, $\nu=1$, and $\kappa = 1$) and by selecting hyperparameters on a separate set of answers, we outperform the existing methods. When choosing hyperparameters, we also have included a boolean flag whether the graph Laplacian should be normalized, as, generally speaking, both the normalized and the standard graph Laplacians can be used with heat and Mat\'ern kernels
\begin{equation}
L_{\mathrm{n}} = \left(D^{+}\right)^{1 / 2} L\left(D^{+}\right)^{1 / 2},
\end{equation}
where $D^{+}$ is the Moore-Penrose inverse of the degree matrix $D$. We observe similar results when analyzing other semantic kernels.

\begin{figure}[t]
\centering
  \centering
  \includegraphics[width=\linewidth]{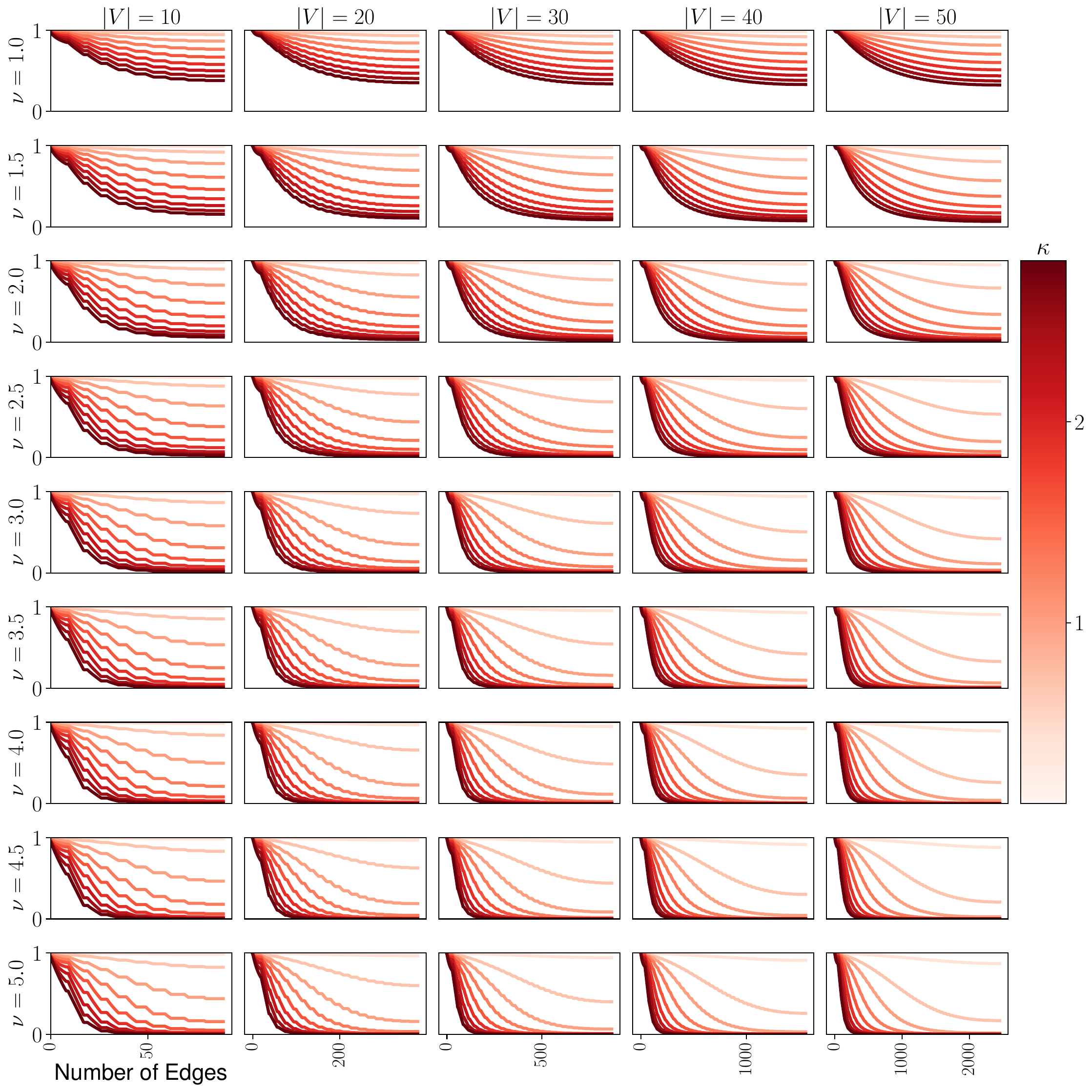}
\caption{Mat\'ern Entropy Convergence Plots.}
\label{fig:app:matern_ecp}
\end{figure}

\begin{figure}[t]
\centering
  \centering
  \includegraphics[width=\linewidth]{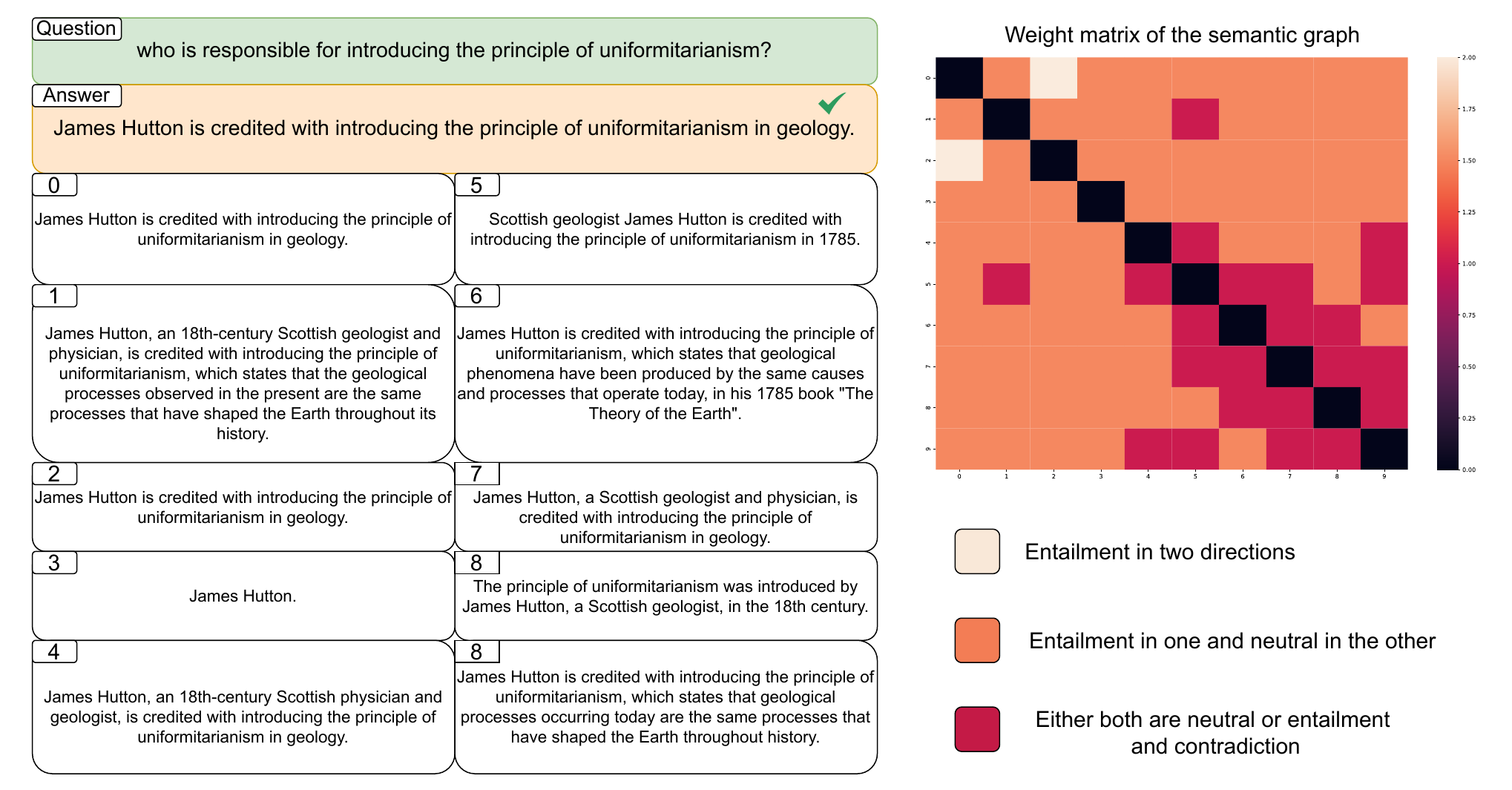}
\caption{Example from NQ.}
\label{fig:app:example}
\end{figure}

\begin{figure}[htbp!]
\centering
\begin{subfigure}{.5\textwidth}
  \centering
  \includegraphics[width=\linewidth]{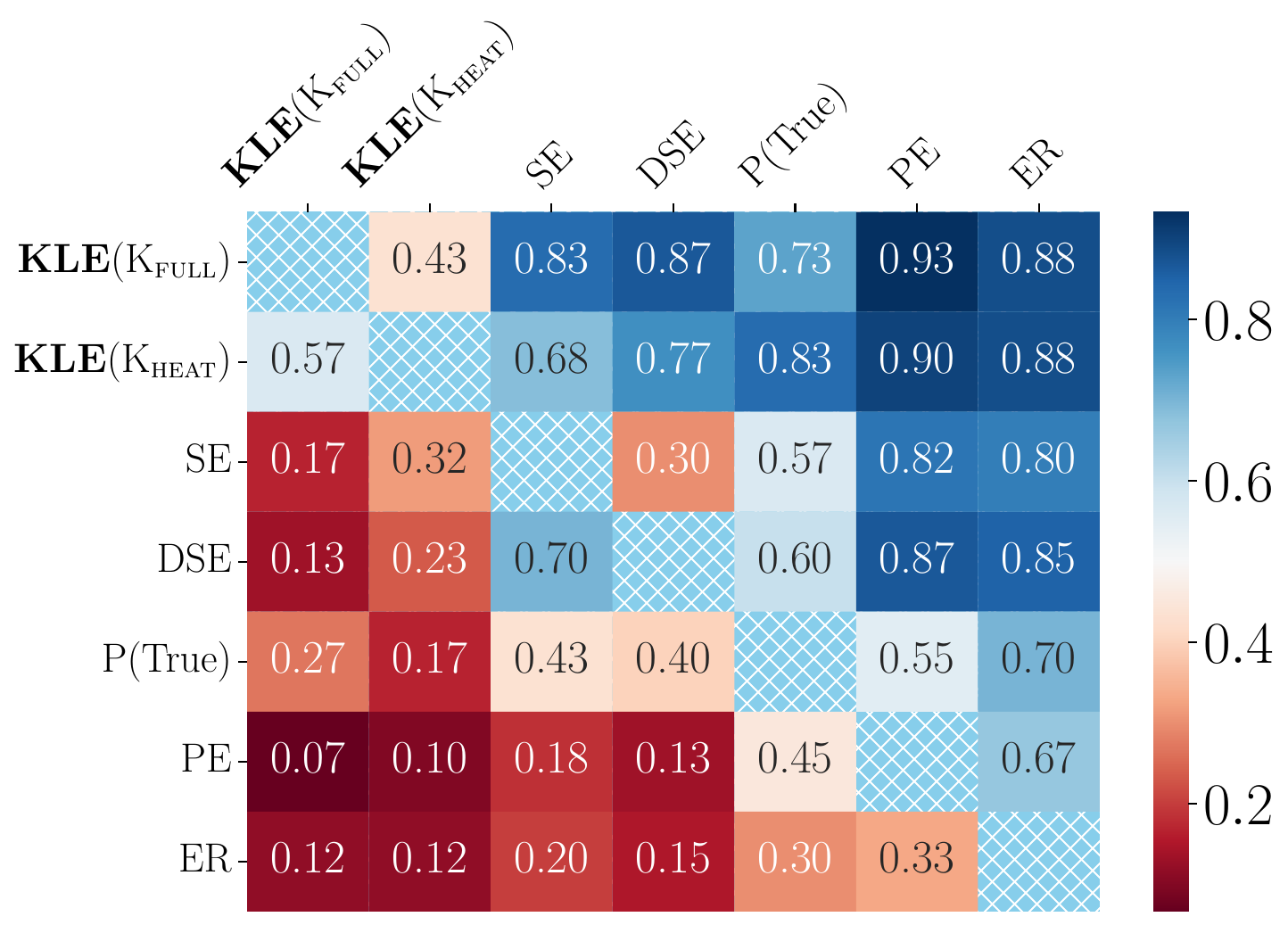}
  \caption{AUROC: Default hyperparameters}
  \label{app:fig:default_AUROC}
\end{subfigure}%
\begin{subfigure}{.5\textwidth}
  \centering
  \includegraphics[width=\linewidth]{images/heatmaps/heatmaps_all_best_AUROC.pdf}
  \caption{AUROC: Validation set hyperparameters}
  \label{app:fig:best_AUROC}
\end{subfigure}
\vspace{1em}
\begin{subfigure}{.5\textwidth}
  \centering
  \includegraphics[width=\linewidth]{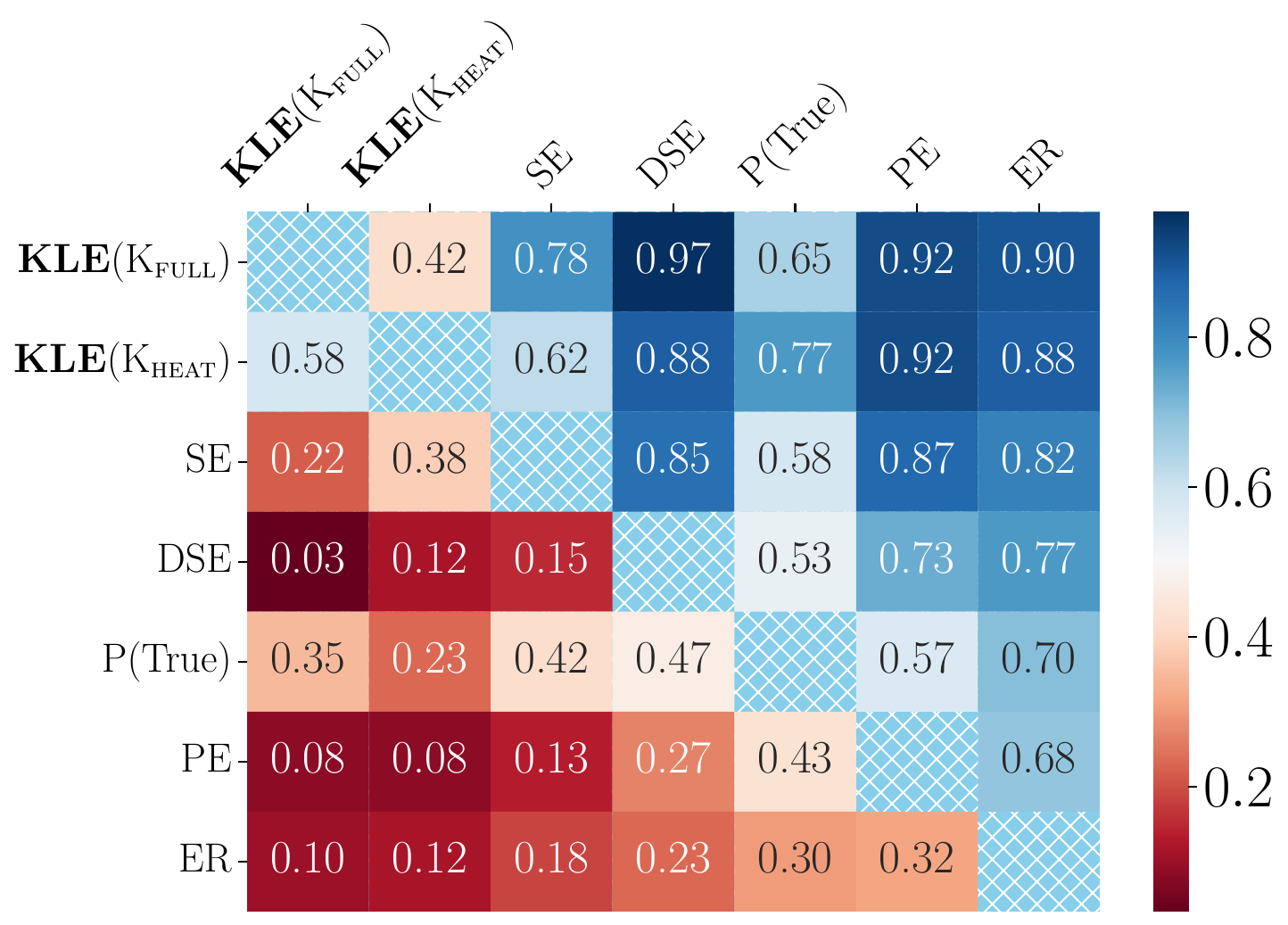}
  \caption{AUARC: Default hyperparameters}
  \label{app:fig:default_AUARC}
\end{subfigure}%
\begin{subfigure}{.5\textwidth}
  \centering
  \includegraphics[width=\linewidth]{images/heatmaps/heatmaps_all_best_area_under_thresholded_accuracy.pdf}
  \caption{AUARC: Validation set hyperparameters}
  \label{app:fig:best_AUARC}
\end{subfigure}
\caption{Summary of \textbf{60} experimental scenarios. Comparing hyperparameters selection strategies. Our methods are labeled $\KLE(\cdot)$.}
\label{app:fig:comparing_hyperparams}
\end{figure}

\begin{wrapfigure}{R}{0.5\textwidth}
\vspace{-10pt}
     \centering
     \includegraphics[width=.55\textwidth]{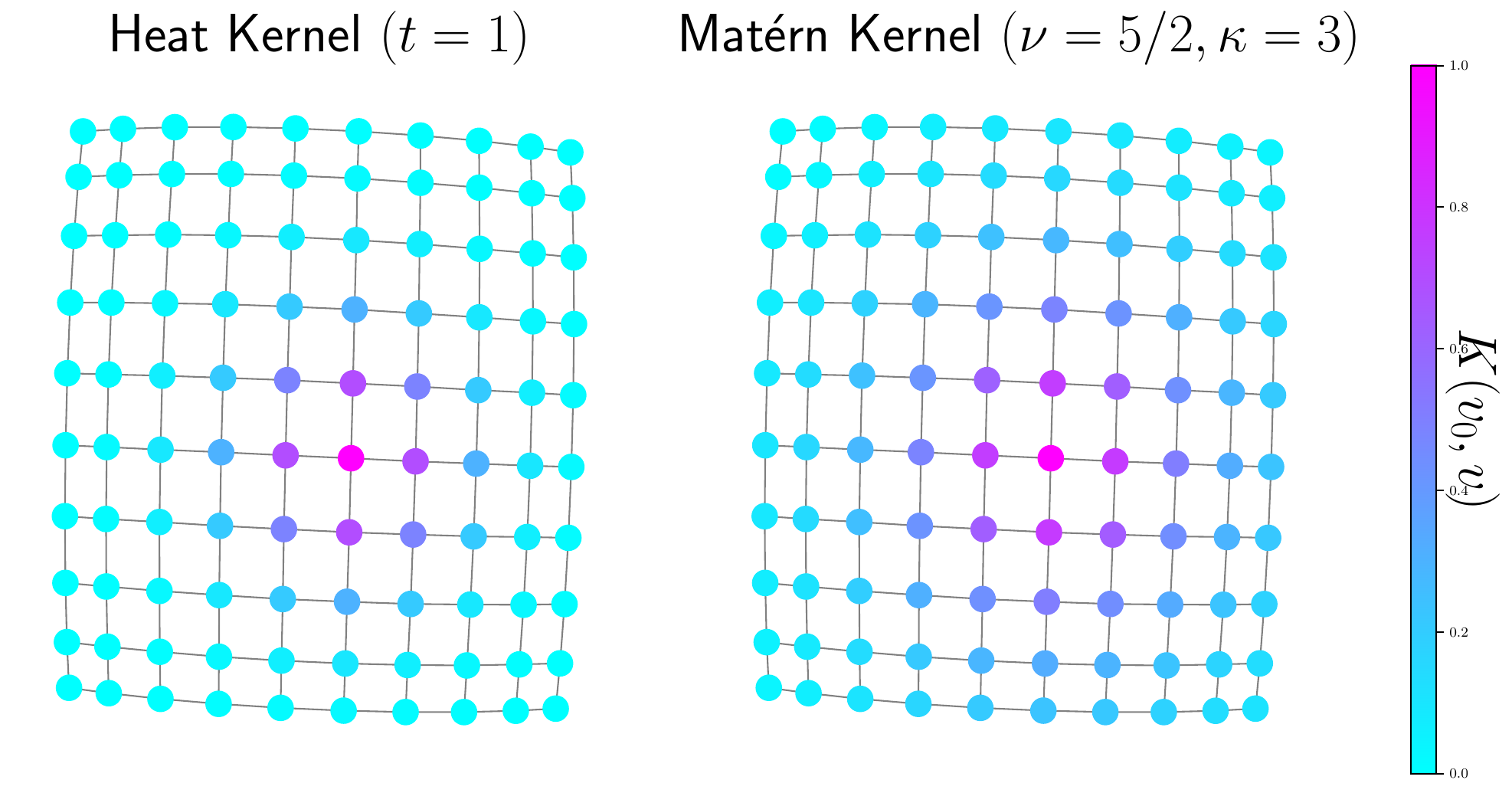}
     \caption{Heat and Mat\'ern kernels visualized on 2-d grid.}
     \label{fig:app:heat_and_matern}
\vspace{-10pt}
\end{wrapfigure}
\begin{figure}

\end{figure}

\paragraph{Prompts.}
We prompt the models to generate full sentences as answers with the following prompt:
    \texttt{Answer the following question in a single brief but complete sentence.}. 
    
Also, we have used the following prompt to check the accuracy of the responses: 

\texttt{We are assessing the quality of answers to the following question: \{\HO{question}\} \textbackslash n
The expected answer is: \{\HO{correct\_answer}\}. \textbackslash n The proposed answer is: \{\HO{predicted\_answer}\} \textbackslash n Within the context of the question, does the proposed answer mean the same as the expected answer? \textbackslash n Respond only with yes or no.\textbackslash n Response:}

Here we mark \HO{placeholders} with the orange color. Or, if several correct answers were provided, we have used the following prompt:

\texttt{We are assessing the quality of answers to the following question: \{\HO{question}\} \textbackslash n
The following are expected answers to this question: \{\HO{correct\_answers}\}. \textbackslash n The proposed answer is: \{\HO{predicted\_answer}\} \textbackslash n Within the context of the question, does the proposed answer mean the same as any of the expected answers? \textbackslash n Respond only with yes or no.\textbackslash n Response:}

\paragraph{Example.}
We visualize an example from the NQ dataset in \cref{fig:app:example}; we have used Llama-2 70B Chat for this example. In order to analyze cases where SE and KLE are inconsistent, we ranked all the answers separately by KLE and SE and found those cases where the difference between indices in the list ranked by KLE and ranked by SE is high. In \cref{fig:app:example}, a model provides the correct answer. However, SE estimates the uncertainty to be high because it can detect only two answers as equal and thus considers the majority of the answers as semantically distinct.  Instead, our method considers more fine-grained relations between the answers and provides better uncertainty estimates (i.e., orange and red cells in the weight matrix). It is an illustrative example of the cases we analyzed. It indicates that the longer and more nuanced the answers are, the more KLE would outperform SE.

\section{Additional Experimental Details}
\label{appendix:additional_exps}
In this section, we provide additional experimental results.

\textbf{Hardware and Resources.} We ran Llama 2 70B models on two NVIDIA A100 80GB GPUs, and the rest of the models on a single NVIDIA A100 80GB. The generation process took from one to seven hours depending on a model for each experimental scenario, and the evaluation additionally took roughly four hours per scenario which can be further optimized by reducing the number of hyperparameters. The project spent more resources due to other experiments. Our experimental pipeline first generates the answers for all the datasets and then computes various uncertainty measures. We did not recompute generations, but in each experimental run we only evaluated uncertainty measures. 

\textbf{Licenses.} We release our code under a clear BSD-3-Clause-Clear. The datasets used in this paper are released under CC BY 2.5 (BioASQ \citep{krithara2023bioasq}), Apache 2.0 (TriviaQA \citep{joshi2017triviaqa}),  CC BY-SA 4.0 (SQuAD \citep{rajpurkar2018know}), MIT (SVAMP, \citep{patel-etal-2021-nlp}), and CC BY-SA 3.0 (NQ \citep{kwiatkowski2019natural}).

\subsection{Models and datasets}

In \cref{fig:app:dataset_details}, we show samples from each dataset we used in the experimental evaluation of our method.
\begin{figure}[htbp!]
     \centering
     \includegraphics[width=\textwidth]{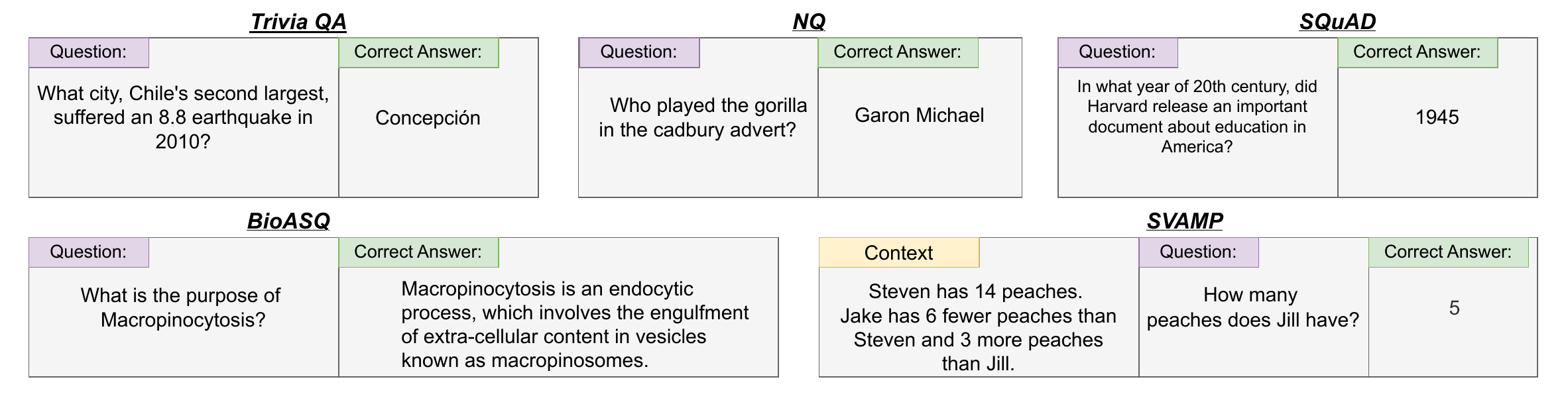}
     \caption{Samples from datasets we use: Trivia QA, NQ, SQuAD, BioASQ, and SVAMP.}
     \label{fig:app:dataset_details}
\end{figure}

\begin{figure}[t]
\centering
  \centering
  \includegraphics[width=\linewidth]{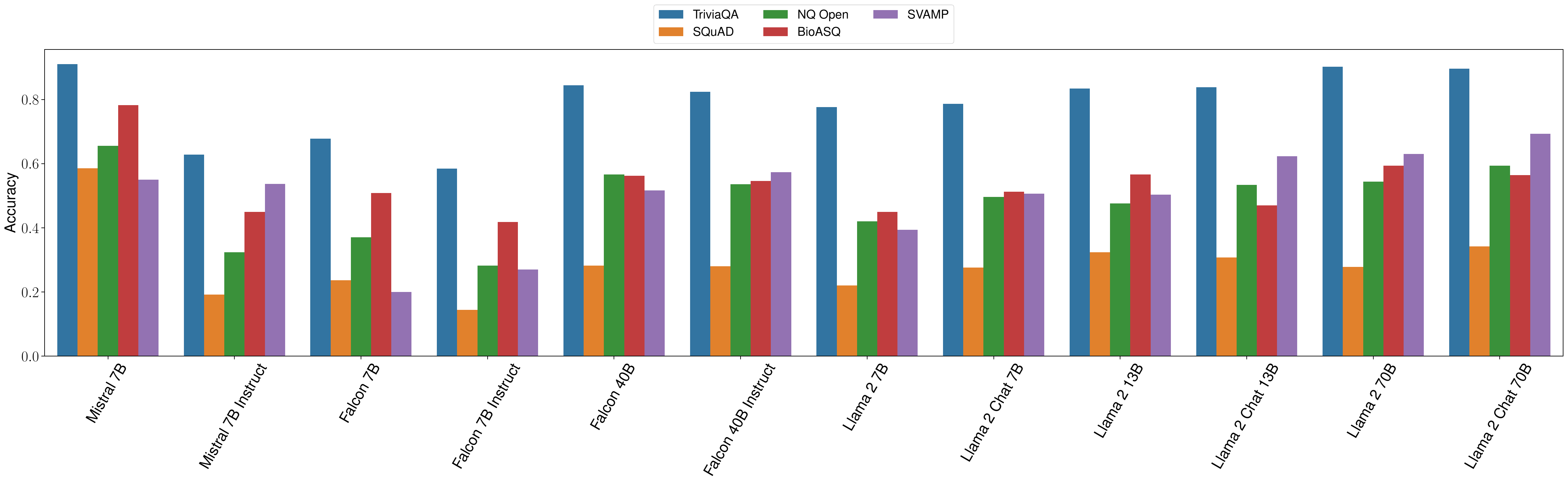}
\caption{Accuracy of the models}
\label{fig:app:accuracy_all}
\end{figure}

Additionally, we demonstrate the accuracy of the models used in the experiments on each dataset in \cref{fig:app:accuracy_all}. As can be seen, we evaluate our method on a diverse set of models with a varying level of accuracy across the tasks at hand. This is especially important for UQ, because UQ methods should perform well for all the models regardless of their downstream effectiveness. 

Real-world applications often involve deploying models with varying degrees of performance, and a robust UQ method should provide reliable uncertainty estimates for all of them. By demonstrating the efficacy of our method across a wide variety of models, we validate its applicability in diverse scenarios. This highlights that our approach can be confidently used in practical settings where model performance can fluctuate.

\subsection{Instruction-tuned and non-instruction tuned models}
Furthermore, we investigate the performance of UQ methods by splitting the set of experimental scenarios into instruction-tuned and non-instruction tuned models. We visualize the splits in \cref{app:fig:instr_vs_non_instr}. Interestingly, our approach significantly outperforms the existing methods when evaluated with instruction-tuned models, and only marginally outperforms when evaluated on non-instruction-tuned models. We can hypothesize that non-instruction tuned models are better calibrated and thus methods based on token-likelihoods make sense whereas with instruction tuning worsens calibration. This hypothesis is also supported by comparison of SE and DSE (DSE significantly outperforms SE on an instruction-tuned split, when AUROC is measured).

\begin{figure}[htbp!]
     \centering
     \includegraphics[width=\textwidth]{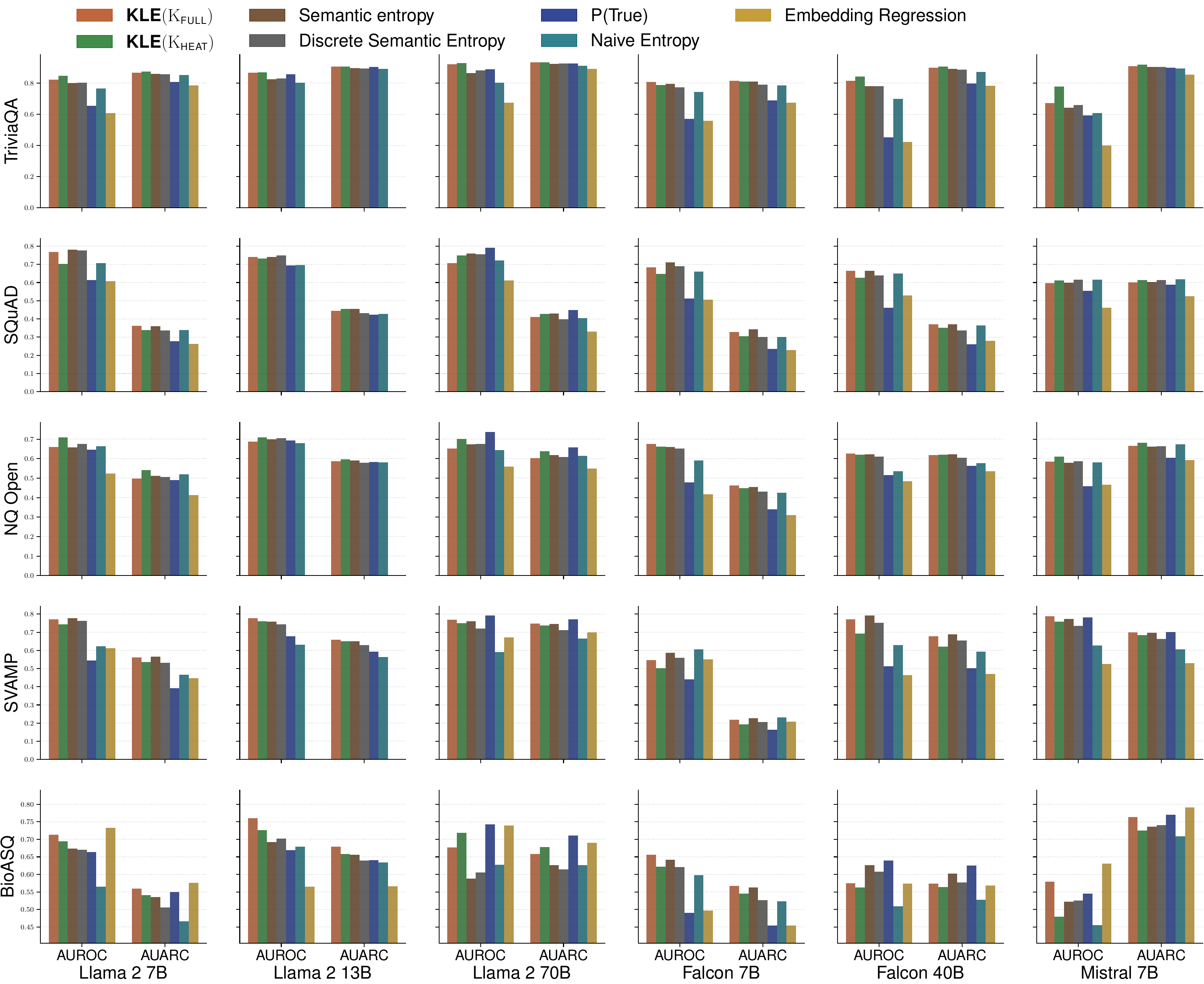}
     \caption{Full results of non-instruction tuned models}
     \label{fig:app:all_results_non_instr}
\end{figure}

\begin{figure}[htbp]
     \centering
     \includegraphics[width=\textwidth]{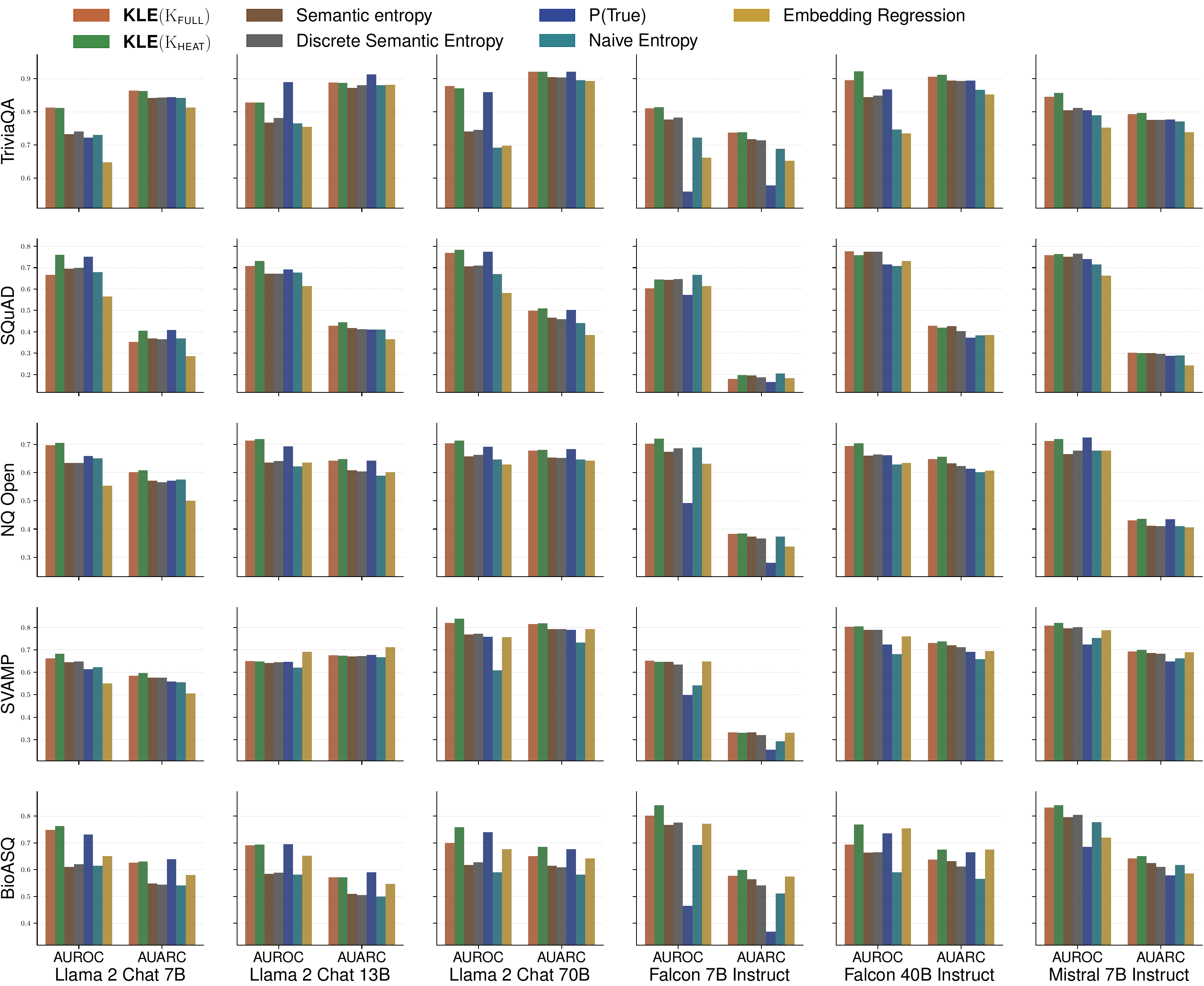}
     \caption{Full results of instruction-tuned models}
     \label{fig:app:all_results_instr}
\end{figure}

\begin{figure}[htbp!]
\centering
\begin{subfigure}{.5\textwidth}
  \centering
  \includegraphics[width=\linewidth]{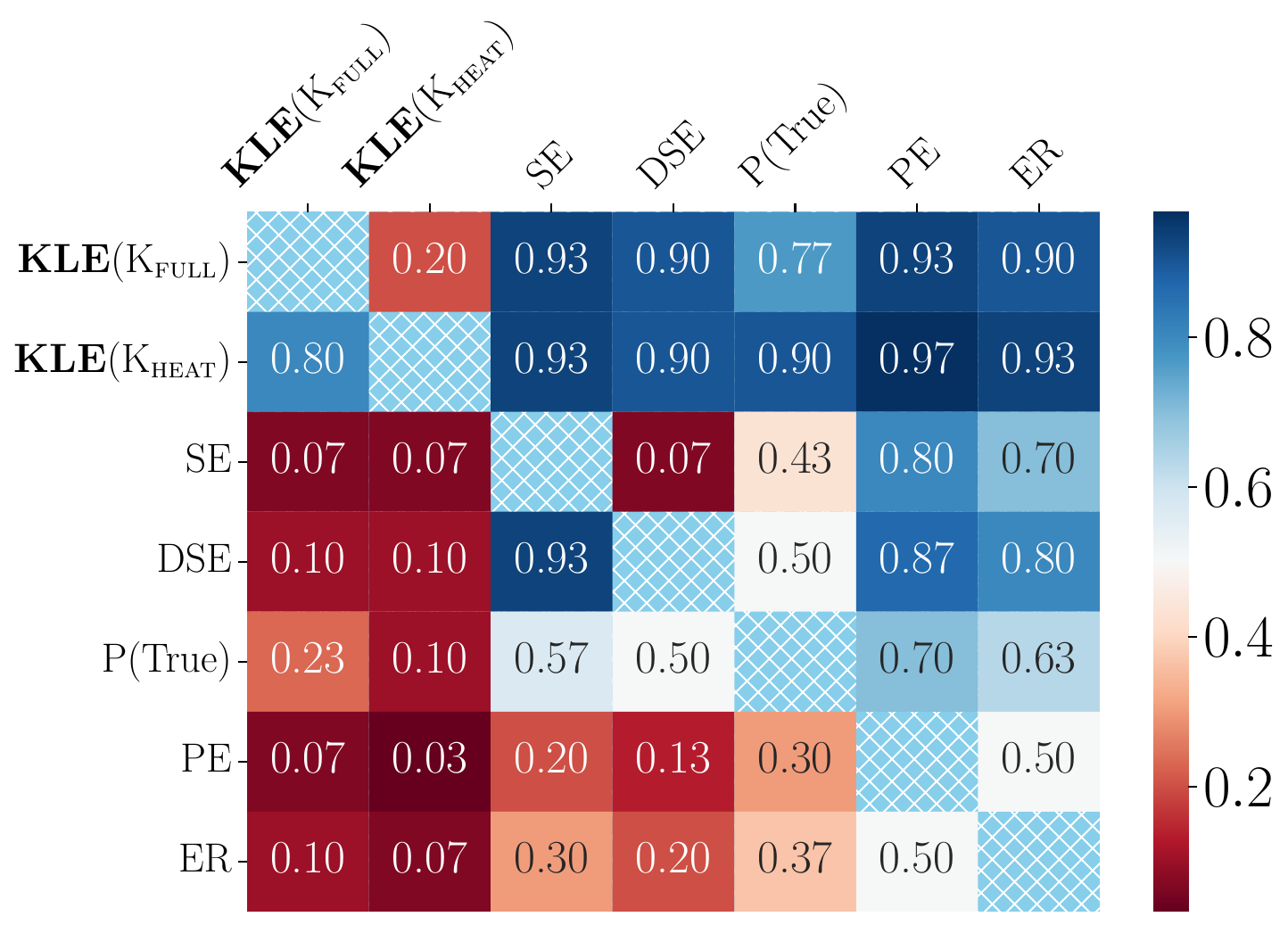}
  \caption{AUROC instruction-tuned models}
  \label{app:fig:instr_auroc}
\end{subfigure}%
\begin{subfigure}{.5\textwidth}
  \centering
  \includegraphics[width=\linewidth]{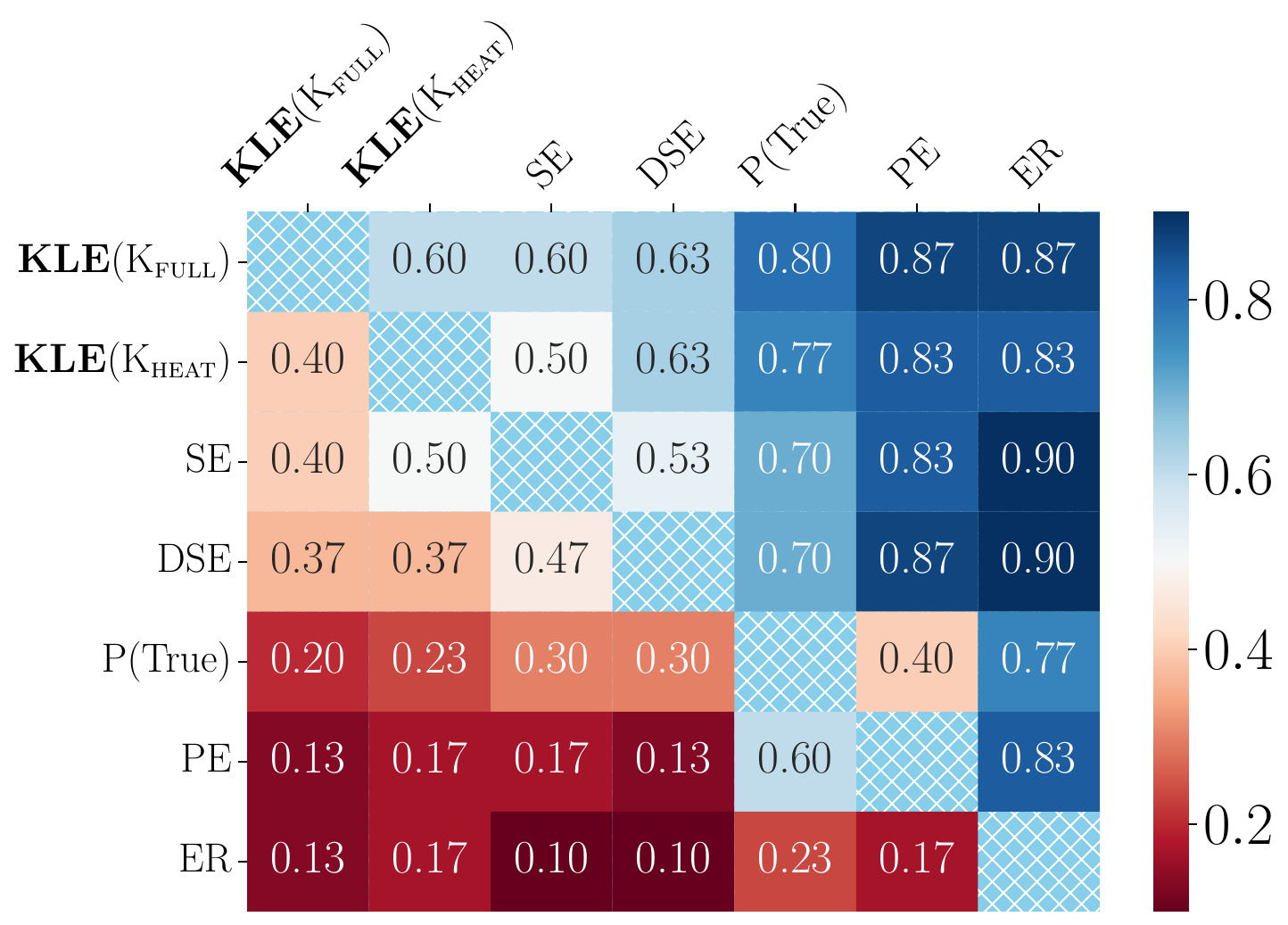}
  \caption{AUROC: non-instrunction tuned models}
  \label{app:fig:heatmap_non_instr_best_auroc}
\end{subfigure}
\vspace{1em}
\begin{subfigure}{.5\textwidth}
  \centering
  \includegraphics[width=\linewidth]{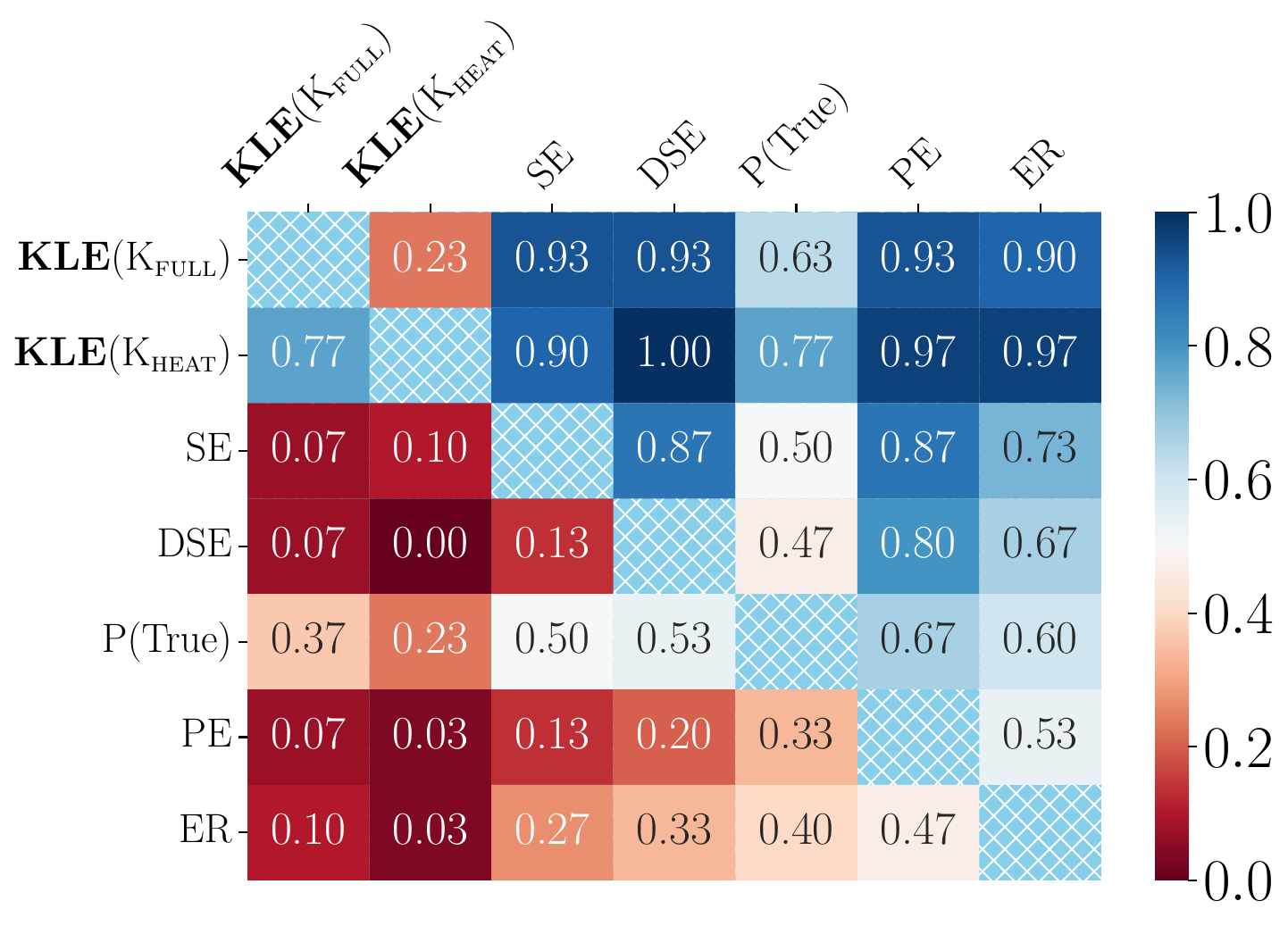}
  \caption{AUARC: instruction-tuned models}
  \label{app:fig:heatmap_instr_best_auarc}
\end{subfigure}%
\begin{subfigure}{.5\textwidth}
  \centering
  \includegraphics[width=\linewidth]{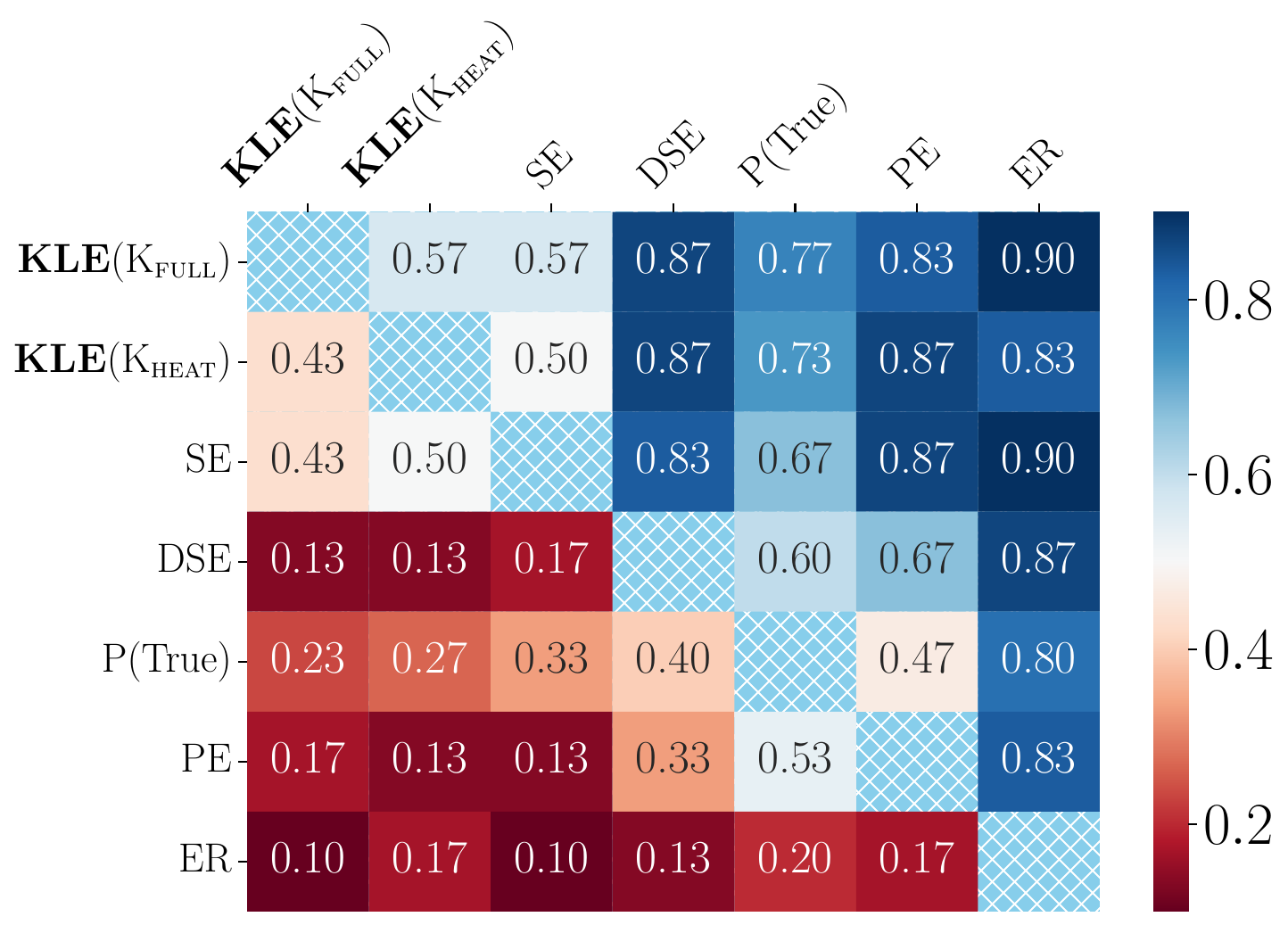}
  \caption{AUARC: non-instrunction tuned models}
  \label{app:fig:heatmap_non_instr_best_auarc}
\end{subfigure}
\caption{Summary of \textbf{60} experimental scenarios. Comparing the results on instruction-tuned and non-instruction tuned models. Our methods are labeled $\KLE(\cdot)$.}
\label{app:fig:instr_vs_non_instr}
\end{figure}

\subsection{Detailed results of UQ}
We provide a detailed comparison of our method with previous uncertainty quantification measures. In \cref{fig:app:all_results_non_instr} and \cref{fig:app:all_results_instr}, we show the results for a wide range of models across five datasets for non-instruction tuned and instruction-tuned models, respectively. We want to note that ER has failed for Llama 2 13B (non-instruction tuned version) for all datasets except BioASQ because training datasets for ER contained samples of only one class. We have assigned zero score to the failed cases.

\section{Additional Notes}

\subsection{Lexical, semantic, and syntactic variability}
\label{appendix:lex_sem_synt}

\begin{table}[h]
\label{appendix:syntactic_vs_semantic_vs_lexical}

\caption{Examples of semantic, syntactic, and lexical variability of a sentence ``Paris is the capital of France.''}
\centering
\adjustbox{max width=1\textwidth}{

\begin{tabular}{c|c|c|c}

 & Semantic Variability & Syntactic Variability & Lexical Variability \\
\hline
Paris is the capital of France. & \makecell{Rome is capital of France \\ Paris is the capital of Italy.} & The capital of France is Paris. & \makecell{France's capital is situated in Paris.\\France's capital city is Paris.} \\

\end{tabular}
}

\end{table}

We resort to the 6-level model of the structure for text analysis proposed in \citep{crystal2018cambridge} to extensively describe aspects of language beyond semantics. This model distinguishes four basic notions for text analysis: medium of transmission, grammar, semantics, and pragmatics. Medium of transmission is irrelevant to the study of language model outputs (however, it becomes relevant for multimodal foundation models that can, for instance, answer a request either with a text or an image); grammar is further divided into the syntax and morphology of the text and semantics into semantics and discourse. Another dimension is pragmatics, or how the text is used. In this work, we focus only on the semantics of the text. However, the method can be extended to other aspects of text analysis. For instance, one can design syntactic or pragmatic kernels. We leave the study of other kernel modalities to future works.

\FloatBarrier

\end{document}